\newif\ifarxiv
\arxivtrue    
\ifarxiv
\documentclass[a4paper]{article}
\else
\documentclass[journal]{IEEEtran}
\fi

\usepackage{cite}
\usepackage[pdftex]{graphicx}

\usepackage{amsfonts,dsfont}
\usepackage[cmex10]{amsmath}
\usepackage{amsthm}
\ifarxiv \usepackage{authblk}
\usepackage[title]{appendix}
\usepackage{pdflscape}
\fi

\usepackage{array}
\usepackage[caption=false,font=footnotesize]{subfig}
\usepackage{fixltx2e}
\usepackage{stfloats}
\usepackage{url}

\usepackage{xspace}                           \usepackage{booktabs,multirow}                         \usepackage{marginnote}

\allowdisplaybreaks

\newtheorem{theorem}             {Theorem}
\newtheorem{lemma}      [theorem]{Lemma}
\newtheorem{corollary}  [theorem]{Corollary}

\newtheorem{definition} [theorem]{Definition}

\newcommand{\fprop}{level function\xspace}

\newcommand{\prob}[1]{\Pr\left(#1\right)}
\newcommand{\probt}[1]{\Pr{}_t\left(#1\right)}
\newcommand{\expect}[1]{E\left[#1\right]}
\newcommand{\expectt}[1]{E_t\left[#1\right]}

\newcommand{\indfe}[1]{\ensuremath{\mathds{1}_{#1}}}
\newcommand{\indf}[1]{\indfe{\{#1\}}}

\newcommand{\Real}{\mathbb{R}}

\DeclareMathOperator{\poly}{poly}
\DeclareMathOperator{\bin}{Bin}
\DeclareMathOperator{\bernoulli}{Bernoulli}
\DeclareMathOperator{\unif}{Unif}

\usepackage{mathrsfs}
\newcommand{\filtuc}[1]        {\mathscr{F}_{#1}}
\newcommand{\filt}[1]        {\filtuc{#1}}

\newcommand{\ab}{\hspace{0.125em}}                        \newcommand{\ie}{\hbox{i.\ab e.}\xspace}                  \newcommand{\eg}{\hbox{e.\ab g.}\xspace}                  \newcommand{\wrt}{\hbox{w.\ab r.\ab t.}\xspace}           \newcommand{\wolg}{\hbox{w.\ab o.\ab l.\ab g.}\xspace}

\newcommand{\bigO}[1]{\ensuremath{\mathcal{O}\left(#1\right)}}

\newcommand{\oneplusoneEA}{($1$+$1$)\ab{}EA\xspace}
\newcommand{\muplusmuEA}{($\mu$+$\mu$)\ab{}EA\xspace}

\newcommand{\muplusoneEA}{($\mu$+$1$)\ab{}EA\xspace}
\newcommand{\muplusoneGA}{($\mu$+$1$)\ab{}GA\xspace}
\newcommand{\onepluslambdacommalambdaGA}{($1$+($\lambda$,$\lambda$))\ab{}GA\xspace}

\newcommand{\sel}{\ensuremath{\mathrm{Sel}}\xspace}
\newcommand{\mut}{\ensuremath{\mathrm{Mut}}\xspace}
\newcommand{\xor}{\ensuremath{\mathrm{Cross}}\xspace}

\newcommand{\psel}{\ensuremath{p_\mathrm{sel}}\xspace}
\newcommand{\pmut}{\ensuremath{p_\mathrm{mut}}\xspace}
\newcommand{\pxor}{\ensuremath{p_\mathrm{xor}}\xspace}

\newcommand{\exchg}{\ensuremath{\mathrm{Exchange}}\xspace}

\newcommand{\genbeta}{\ensuremath{\zeta}\xspace}

\usepackage{algorithm,algorithmic}

\newcommand{\leadingones}{\text{\sc LeadingOnes}\xspace} \newcommand{\lo}{\text{\sc Lo}\xspace}
\newcommand{\onemax}{\text{\sc OneMax}\xspace} \newcommand{\om}{\text{\sc Om}\xspace}
\newcommand{\jump}{\text{\sc Jump}\xspace} \newcommand{\unimodal}{\text{\sc Unimodal}\xspace} \newcommand{\linear}{\text{\sc Linear}\xspace} \newcommand{\sortinginv}{\text{\sc Sorting\textsubscript{Inv}}\xspace} \newcommand{\inv}{\text{\sc Inv}\xspace} \newcommand{\rr}{\text{\sc RoyalRoad}\xspace} \newcommand{\maxsat}{\text{\sc Max-SAT}\xspace} \newcommand{\maxcut}{\text{\sc Max-CUT}\xspace}

\usepackage{enumitem}

\usepackage{color}

\begin{document}

\title{Level-Based Analysis of Genetic Algorithms and Other Search Processes}

\ifarxiv
\author[1]{Dogan~Corus\footnote{Dogan Corus' contributions to this
    paper was made while he was a PhD student at the University of Nottingham.}}
\author[2]{Duc-Cuong~Dang}
\author[3]{Anton~V.~Eremeev}
\author[2]{Per~Kristian~Lehre}
\affil[1]{University of Sheffield, United Kingdom}
\affil[2]{University of Nottingham, United Kingdom}
\affil[3]{Omsk Branch of Sobolev Institute of Mathematics, Russia}
\else
\author{Dogan~Corus,
        Duc-Cuong~Dang,
        Anton~V.~Eremeev,
        and Per~Kristian~Lehre
\thanks{D. Corus is with the Algorithms Research Group         at the Department of Computer Science,
        University of Sheffield,
        United Kingdom.}
\thanks{D.-C. Dang and P. K. Lehre are with ASAP Research Group
        at the School of Computer Science, University of Nottingham,
        United Kingdom.}
\thanks{A. V. Eremeev is with the Chair of Applied and Computational Mathematics,
        Omsk State University
        and  Laboratory of Discrete Optimization,
        Omsk Branch of Sobolev Institute of Mathematics,
        Russia.}
\thanks{Manuscript received         ; revised }}
\markboth{IEEE TRANSACTIONS ON EVOLUTIONARY COMPUTATION, AUTHOR-PREPARED MANUSCRIPT}
{IEEE TRANSACTIONS ON EVOLUTIONARY COMPUTATION, AUTHOR-PREPARED MANUSCRIPT}
\fi

\maketitle

\begin{abstract}
Understanding how the time-complexity of evolutionary algorithms
(EAs) depend on their parameter settings and characteristics of
fitness landscapes is a fundamental problem in evolutionary
computation.  Most rigorous results were derived using a handful of
key analytic techniques, including drift analysis.  However, since few
of these techniques apply effortlessly to population-based EAs, most
time-complexity results concern simplified EAs, such as the (1+1) EA.

This paper describes the \emph{level-based theorem}, a new technique
tailored to population-based processes.  It applies to any
non-elitist process where offspring are sampled independently from a
distribution depending only on the current population.  Given
conditions on this distribution, our technique provides upper
bounds on the expected time until the process reaches a target state.

We demonstrate the technique on several pseudo-Boolean functions,
the sorting problem, and approximation of optimal solutions in
combinatorial optimisation. The conditions of the theorem are
often straightforward to verify, even for Genetic Algorithms and
Estimation of Distribution Algorithms which were considered highly
non-trivial to analyse.  Finally, we prove that the theorem is
nearly optimal for the processes considered. Given the information
the theorem requires about the process, a
much tighter bound cannot be proved.

\end{abstract}

\ifarxiv
\else
\IEEEpeerreviewmaketitle
\fi

\section{Introduction}\label{sec:intro}

\ifarxiv 
The
\else
\IEEEPARstart{T}{he} 
\fi 
theoretical understanding of Evolutionary
Algorithms (EAs) has advanced significantly over the last decade.  A
contributing factor for this success may have been the strategy to
analyse simple settings before proceeding to more complex scenarios,
while at the same time developing appropriate analytic techniques.  In
particular, much of the work assumed a population size of one, and no
crossover operator. Current approaches to analysing evolutionary
algorithms often rely on one or more of these simplifying assumptions.

This paper presents a general-purpose technique to analyse a large
class of search heuristics involving non-overlapping populations.
In our framework, each individual of the current population is
independently sampled from the same distribution over the search
space parametrised by the previous generation. A similar modelling
of the search process first appeared in \cite{bib:vose99} to
analyse Genetic Algorithms~(GAs) however as far as we know, mainly
results at the limit of infinite population were established. In
this paper, we give the following general result for finite
populations. Given some requirement on the upper tails of this
distribution over an ordered partition of the search space and a
minimum requirement on the population size, our method will
guarantee an upper bound on the expected runtime to reach the last
set of the partition.

Particularly, the partition of the search space is similar to the
well-known fitness-level technique \cite{bib:Wegener2002} to
analyse \emph{elitist} EAs, however at our general level of
describing the search process, the traditional requirement on a
\emph{fitness-based} (this will be properly defined later on)
partition is no longer required.
Applications of the fitness-level technique itself are widely
known in the literature for classical \emph{elitist} EAs
\cite{bib:Wegener2002}. One of the first examples of using this
technique in the analysis of non-elitist EAs is~\cite{Eremeev2000}
where lower and upper bounds on the expected proportions of the
population above certain fitness levels were found.

Related to our work, early research on analysing population-based
EAs often ignored recombination operators. The family tree technique
was introduced in \cite{Witt2006FamilyTree} to analyse the
\muplusoneEA. The performance of the \muplusmuEA for different settings
of the population size was conducted in \cite{bib:He2002} using Markov
chains to model the search processes, and in \cite{bib:Chen2009c}
using a similar argument to fitness-levels. The analysis of parallel
EAs in \cite{bib:Lassig2010} also made use of the fitness-levels
argument. The inefficiency of standard fitness proportionate selection
without scaling was shown in \cite{bib:Neumann2009} and in
\cite{bib:l11} using drift analysis \cite{bib:Hajek1982}. In the
recently introduced switch analysis, the progress of the EA is
analysed relative to an easier understand reference process
\cite{Yu2015}. When the method applies, bounds on the runtime of the
reference process can be translated into bounds on the original
process. In current applications of this method, the reference process
is RLS$^=$, a simple local search algorithm. It remains to be seen how
such simple search heuristics can approximate the population dynamics
of complex EAs.

Over the recent years, runtime analysis of EAs with recombination,
often referred to as Genetic Algorithms, has been subject to
increasing interest. Generalising the work in
\cite{bib:Neumann2009}, \cite{bib:oliveto14,bib:oliveto15} showed
that the Simple Genetic Algorithm~\cite{bib:vose99} is inefficient on
\onemax, even when crossover is used.  A long sequence of work has
attempted to show that enabling crossover can reduce the runtime.
It has been shown that adding crossover to the \muplusoneEA can
decrease the runtime on the \jump problem, however only for small
crossover probabilities \cite{bib:jw02,KotzingSudholtTheile2011}.
For realistic crossover probabilities, it was shown that
\muplusoneGA can decrease the runtime by an exponential factor on
instances of an FSM testing problem, however this result assumes a
deterministic crowding diversity mechanism \cite{LehreYaoXOR}. With
the same setting on the standard \onemax function, crossover was
shown to lead to a constant speedup in \cite{bib:s12}, however
this result assumed a tailored selection mechanism. 
Seeking the construction of an efficient \emph{unbiased} algorithms for
\onemax, \cite{bib:doerr15a} introduced the \onepluslambdacommalambdaGA
and showed a significant speed up with the right choices of the 
offspring population size \cite{bib:doerr15b,bib:doerr15c}. Another 
modified GA, but this time \emph{non-elitist}, was introduced in
\cite{bib:Prugel-Bennett2015}, and its efficiency was proved on
the noisy version of \onemax function.
In~\cite{bib:ms15}, a runtime result is proposed for a class of
convex search algorithms, including some non-elitist GAs with gene
pool recombination and no mutation, on the so-called 
\emph{quasi-concave fitness landscapes}.
As a corollary, 
it has been shown that the convex
search algorithm has~$\bigO{n \log n}$ expected runtime
on \leadingones.
Those results gave the impression that adjustments or modifications 
to the standard setting of GAs, here elitist, are often required to 
illustrate the advantage of crossover. 
Until recently, it has been shown that 
the standard \muplusoneGA without too low crossover probability
has a speed up of $\Omega(n/\log(n))$ on the \jump problem compared to 
mutation-only algorithms \cite{DangEtAl2016Jump}.

Significant progress in developing and understanding a formal
model of canonical GA and its generalisations was was made
in~\cite{bib:vose99} using dynamical systems. In particular it
turned out that the behaviour of the dynamical systems model is
closely related to the local optima structure of the problem in
the case of binary search spaces~\cite{bib:vw95}. Most of the
findings in~\cite{bib:vw95, bib:vose99} apply to the infinite
population case, so it is not clear how these results can be used
in runtime analysis of EA.

A relatively new paradigm in Evolutionary Computation is Estimation of
Distribution Algorithm (EDA) \cite{bib:Larranaga2002}. Unlike traditional EAs
which use explicit genetic operators such as mutation,
recombination and selection, an EDA builds a probabilistic model
for sampling new search points so that the probability of creating an optimal solution
via sampling eventually increases high. The algorithm often starts with a specific
probabilistic model, which is gradually updated through selected solutions of
intermediate samplings.
Over the recent years, many variants of EDAs have been proposed,
along with theoretical investigations on their convergence and
scalability, \eg 
\cite{bib:Gonzalez2000,
bib:Muhlenbein1997,
bib:Pelikan2002,
bib:Shapiro2005,
bib:Zhang2004b}.
However, rigorous runtime analysis results for this particular
class of algorithms on discrete domain are still sparse.
The first analysis of this kind was conducted in \cite{bib:Droste2006}
for the compact Genetic Algorithm (cGA) \cite{bib:Harik1999} on linear
functions. Further work showed that this algorithm can be resilient to
noise \cite{2015ISAAC_Noise}.

Another simple EDA is the Univariate Marginal Distribution
Algorithm (UMDA) which was proposed in \cite{bib:Muhlenbein1996},
and analysed in a series of papers \cite{bib:Chen2009a,bib:Chen2007,
bib:Chen2009b,bib:Chen2010}.
With the $n$-dimensional Hamming cube as search space, each generation of UMDA consists of
  first sampling a population of solutions based on a vector
  $(p_i)_{i\in[n]}$ of frequencies, \ie assuming independence between
  bit positions, then summarising the selected solutions as the new
  sampling vector for the next generation.
The initial result of
\cite{bib:Chen2007} was provided for {\sc LeadingOnes} and a
harder function known as {\sc TrapLeadingOnes} under the so-called
``no-random-error'' assumption and with a sufficiently large
population. The assumption was lifted due to the technique
presented in \cite{bib:Chen2010}. Nevertheless, the analysis
assumes an unrealistically large population size, leading in
overall to a too high bound on the expected runtime. Note also
that there are two versions of the algorithm based on whether or
not margins are imposed to~$p_i$, the difference between the two
in terms of time complexity for various functions are discussed in
\cite{bib:Chen2009b}. More interestingly, \cite{bib:Chen2009a}
showed that UMDA without margins beats the $(1+1)$~EA on a
particular function called {\sc SubString}. However, it is not
recommended to use UMDA without margins in practice, as the
algorithm can always end up with a premature convergence.

In this paper, we show that all non-elitist EAs with or without crossover,
and even UMDA can be cast and analysed in the same framework.
A preliminary version of the paper was communicated in \cite{bib:Corus2014}.
This followed the line of work dated back
to the introduction of a fitness-level technique to analyse \emph{non-elitist} EAs
with linear ranking selection \cite{LehreYaoTEVC2012}, later on generalised to
many selection mechanisms and \emph{unary variation operators} \cite{bib:l11},
with a refined result
in \cite{bib:dl16}.
The original fitness-level technique and its generalisation to
the level-based technique have already found a number of applications,
including analysis of EAs in uncertain environments, such as partial
information 
\cite{bib:dl16},
noisy fitness functions
\cite{DangLehre2015Noise}, and dynamic fitness functions
\cite{DangJansenLehre2015Dynamic}. 
It has also been applied to
analyse the runtime of complex algorithms,
such as 
GAs
for shortest paths
\cite{CorusLehre2015MIC}, 
EDAs
\cite{DangLehre2015}, and self-adaptive 
EAs \cite{DangLehre2016SelfAdaptationArxiv}.

The present work improves the main result of \cite{bib:Corus2014} in
many aspects. A more careful analysis of the population dynamics leads
to a much tighter expression of the runtime bound compared to
\cite{bib:Corus2014}, immediately implying improved results in the
previously mentioned applications. In particular, the leading term in
the runtime is improved by a factor of $\Omega(\delta^{-3})$, where
$\delta$ characterises how fast good individuals can populate the
population. This significantly improves the results of
\cite{DangLehre2015Noise} and \cite{bib:dl16} 
concerning noisy
optimisation, for which $\delta$ is often very small (\eg $1/n$).  We
also provide guideline how to use the theorem to analyse the runtime
of non-elitist processes. Selected examples are given for the cases of
GAs and UMDA in optimising standard pseudo-Boolean functions, a simple
combinatorial problem, and in searching for local optima of NP-hard
problems. Furthermore, we prove that the level-based theorem is close
to optimal for the class of evolutionary processes it applies to.

The paper is structured as follows. In Section~\ref{sec:main}, we
first present the general scheme of the algorithms covered by the main
result of the paper and show how the GAs fit as special cases into
this scheme. The main result of the paper is then presented along with
a set of corollaries tailored to specific cases.  Applications of the
main result to different GAs are considered in
Section~\ref{sec:app}. The section starts with runtime analysis of the
Simple Genetic Algorithm on standard functions, followed by the
results for combinatorial optimisation problems, finally, the main
theorem is again applied to analyse an Estimation of Distribution
Algorithm. Section~\ref{sec:tightness} considers the tightness of the
level-based theorem. Finally, concluding remarks are given in
Section~\ref{sec:concl}.

\section{Main result}\label{sec:main}

\subsection{Abstract algorithmic scheme}\label{sec:algo}

We consider population-based algorithms at a very abstract level in
which fitness evaluations, selection and variation operations, which
depending on the current population $P$ of size $\lambda$, are
represented by a distribution~$D(P)$ over a finite set
$\mathcal{X}$. More precisely, the current population $P$ is
a vector $(P(1),\dots,P(\lambda))$ where $P(i) \in \mathcal{X}$
for each $i \in [\lambda]$. $D$ is a mapping from
$\mathcal{X}^\lambda$ into the space of probability distributions
over $\mathcal{X}$. The next generation is obtained by sampling
each new individual independently from $D(P)$. This scheme is
summarised in Algorithm~\ref{algo:Algorithm1}. Here and below, for
any positive integer~$n$, we define~$[n] := \{1, 2,...,n\}$.

\begin{algorithm}  \begin{algorithmic}[1]
    \REQUIRE ~\\
        Finite state space $\mathcal{X}$, and population size $\lambda\in\mathbb{N}$,\\
        Mapping $D$ from~$\mathcal{X}^{\lambda}$ to the space of prob. dist. over $\mathcal{X}$.\\
        Initial population $P_0\in \mathcal{X}^\lambda$.\\
    \FOR{$t=0,1,2,\dots$ until termination condition met}
    \STATE Sample $P_{t+1}(i) \sim D(P_t)$ independently \\ for each $i\in[\lambda]$
    \ENDFOR
  \end{algorithmic}
  \caption{Population-based algorithm.}   \label{algo:Algorithm1}
\end{algorithm}

A scheme similar to Algorithm~\ref{algo:Algorithm1} was studied
in~\cite{bib:vose99}, where it was called \emph{Random Heuristic
Search} with \emph{an admissible transition rule}. Some examples
of such algorithms are Simulated Annealing (more generally any
algorithm with the population composed of a single individual),
Stochastic Beam Search \cite{bib:vose99}, Estimation of
Distribution Algorithms such as the Univariate Marginal
Distribution Algorithm \cite{bib:Chen2009a} and the Genetic
Algorithm~\cite{bib:Gold89}. The previous studies of the framework
were often limited to some restricted settings
\cite{bib:oliveto13} or mainly focused on infinite populations
\cite{bib:vose99}. In this paper, we are interested in finite
populations and develop a general method to deduce the expected
runtime of the search processes defined in terms of \emph{number
of produced search points}. This can be translated to the
\emph{number of evaluations} once a specific algorithm is
instantiated and the optimisation scenario is specified (\eg see
\cite{bib:dl16}).

We illustrate the general scheme of
Algorithm~\ref{algo:Algorithm1} on the example of GA, which is
Algorithm~\ref{algo:GA}.
\begin{algorithm}  \begin{algorithmic}[1]
    \REQUIRE ~\\
        Finite state space $\mathcal{X}$,\\
        Operators: $\sel$, $\xor$ and $\mut$,\\
        Population size $\lambda\in\mathbb{N}$
          and recombination rate $p_{\mathrm{c}} \in [0,1]$.\\
    \STATE $P_0\sim \unif(\mathcal{X}^\lambda)$
    \FOR{$t=0,1,2,\dots$ until termination condition met}
        \FOR{$i=1$ to $\lambda$} \label{algo:GA:inner-loop}
        \STATE $u := P_t(\sel(P_t))$, $v := P_t(\sel(P_t))$.\label{algo:GA:start-iter}
        \STATE $x := \begin{cases}
                 \xor(u,v)        & \text{ with prob. } p_\mathrm{c},\\
                                  u & \text{ with prob. } 1 - p_\mathrm{c}.                  \end{cases}$\label{algo:GA:crossover}
        \STATE $P_{t+1}(i) := \mut(x)$.\label{algo:GA:end-iter}
        \ENDFOR
    \ENDFOR
  \end{algorithmic}
  \caption{Genetic algorithm.}
  \label{algo:GA}
\end{algorithm}
The term Genetic Algorithm is often
applied to EAs that use recombination operators with some a priori
chosen probability~$p_\mathrm{c} > 0$.
Here the standard operators of GA are formally represented by
transition matrices:
\begin{itemize}
  \item $\psel:[\lambda]\times\mathcal{X}^{\lambda}\to [0, 1]$
represents selection operator $\sel\colon
\mathcal{X}^\lambda \rightarrow [\lambda]$ which is randomised,
where $\psel(i|P_t)$ is the probability of selecting the~$i$-th
individual from population~$P_t$. This probability can depend on
the search point~$P_t(i)$, its relationship to the other search
points in $P_t$, and their mappings to fitness values by a
function $f\colon \mathcal{X} \rightarrow \Real$ that the
algorithm aims to optimise. Throughout the paper, we assume w.l.o.g.
the maximisation of $f$.
 \item $\pmut :
{\mathcal X} \times {\mathcal X} \to [0, 1]$, where~$\pmut(y|x)$
is the probability of mutating $x\in {\mathcal X}$ into~$y \in
{\mathcal X}$ by a randomised mutation operator~$\mut\colon
\mathcal{X} \rightarrow \mathcal{X}$.
 \item $\pxor \colon {\mathcal X}\times {\mathcal X}^2
\to [0, 1]$, where~$\pxor(x|u,v)$ is the probability of
obtaining~$x$ as a result of randomised crossover operator (or
recombination) between~$u,v \in {\mathcal X}$. In what follows,
crossover is denoted by~$\xor\colon \mathcal{X}\times\mathcal{X}
\rightarrow \mathcal{X}.$
\end{itemize}
Clearly, conditioned on the current population $P_t$, each
individual $P_{t+1}(i)$ of the next generation is independently
sampled from the same distribution which is parametrised by $P_t$. Thus,
lines~\ref{algo:GA:start-iter}-\ref{algo:GA:end-iter} of
Algorithm~\ref{algo:GA} can be summarised as $P_{t+1}(i) \sim
D(P_t)$ for some $D$ induced by the genetic operators $\sel$,
$\xor$ and $\mut$. The algorithm fits perfectly in the scheme
of Algorithm~\ref{algo:Algorithm1}.

\subsection{Level-based theorem}\label{sec:result}

This section states the main result of the paper, a general technique
for obtaining upper bounds on the expected runtime of any process that
can be described in the form of Algorithm~\ref{algo:Algorithm1}. We
use the following notation. The natural logarithm is denoted
by~$\ln(\cdot)$.
Suppose that for some~$m$ there is an ordered
partition of~$\mathcal{X}$ into subsets $(A_1,\dots,A_{m})$
called {\em levels},
we define $A_{\ge j}:=\cup_{i=j}^{m} A_i$, \ie the union of all
levels above level $j$.
An example of a partition is the \emph{canonical} partition,
where each level regroups solutions having the same fitness value
(see e.g.~\cite{bib:l11}). This partition is classified as
\emph{fitness-based} or $f$-based, if  $f(x)<f(y)$ for all $x \in
A_{j}$, $y \in A_{j+1}$ and all $j \in [m-1]$. As a result of the
algorithmic abstraction, our main theorem is not limited to this
particular type of partition.
Let~$P\in\mathcal{X}^\lambda$ be a population vector of a finite
number $\lambda\in\mathbb{N}$ of individuals.
Given any subset $A \subseteq \mathcal{X}$, we
write~$|P \cap A| := |\{i \mid P(i) \in A\}|$ to denote the number
of individuals in population $P$ that belong to the subset $A$.

\begin{theorem}\label{thm:general-fitness-levels}
Given a partition
    $(A_1,\dots,A_{m})$ of $\mathcal{X}$, define $T :=
    \min\{t\lambda \mid |P_t\cap A_{m}|>0\}$ to be the first point
    in time that elements of $A_{m}$ appear in $P_t$ of Algorithm
    \ref{algo:Algorithm1}.
    If there
    exist $z_1,\dots,z_{m-1},\delta \in(0,1]$,
    and $\gamma_0 \in (0,1)$
    such that
    for any population $P\in \mathcal{X}^\lambda$,
  \begin{description}[noitemsep,leftmargin=3em]
  \item[(G1)]
  for each level $j \in [m-1]$, 
  if
      $|P\cap A_{\geq j}  | \geq \gamma_0\lambda$
            then
\[
    \displaystyle \Pr_{y\sim D(P)}\left( y\in A_{\ge j+1}\right)
      \geq z_j,
\]
  \item[(G2)]
  for each level $j \in [m-2]$, and all $\gamma\in(0,\gamma_0]$
  if
      $|P\cap A_{\ge j}  |  \geq \gamma_0\lambda$ and
      $|P\cap A_{\ge j+1}|  \geq   \gamma\lambda$ then
\[
    \Pr_{y\sim D(P)}\left( y\in A_{\ge j+1}\right)
      \geq (1+\delta)\gamma,
\]
  \item[(G3)] and the population size $\lambda\in\mathbb{N}$ satisfies\[
\lambda \geq \left(\frac{4}{\gamma_0\delta^2}\right)
\ln\left(\frac{128m}{z_*\delta^2}\right) \text{ where }
z_*:=\min_{j\in[m-1]} \{z_j\},
\]
  \end{description}
  then
            \begin{align*}
    \expect{T}
     & \leq \left(\frac{8}{\delta^{2}}\right)
            \sum_{j=1}^{m-1}
            \left(\lambda\ln\left(\frac{6\delta\lambda}{4+z_j\delta\lambda }\right)+\frac{1}{z_j}\right).
\end{align*}
\end{theorem}

Informally, the two first conditions require a relationship
between the current population $P$ and the distribution $D(P)$ of
the individuals in the next generation:
Condition (G1) demands that the probability of creating an individual at level
$j+1$ or higher is at least $z_j$ when some fixed portion $\gamma_0$
of the population has reached level $j$ or higher.
Furthermore, if the number of individuals at level $j+1$ or higher
is at least $\gamma\lambda>0$, condition (G2) requires that their
number tends to increase further, e.g. by a multiplicative factor of
$1+\delta$.
Finally, (G3) requires a sufficiently large population size. When all conditions
are satisfied, an upper bound on the expected time for the algorithm to create
an individual in $A_{m}$ can be guaranteed.

We suggest to follow the five steps below when applying the level-based
theorem.
\begin{enumerate}
\item Identify a partitioning of the search space which reflects
  the ``typical'' progress of the population towards the target set
  $A_{m}$.
\item Find parameter settings of the algorithm and corresponding
  parameters $\gamma_0$ and $\delta$ of the theorem, such that condition (G2) can be
  satisfied. It may be necessary to adjust the partitioning of the
  search space.
\item For each level $j\in [m-1],$ estimate lower bounds $z_j$ such that
  condition (G1) holds.
\item Determine the lower bound on the population size $\lambda$ in
  (G3) using the parameters obtained in the previous steps.
\item Once all conditions are satisfied, compute the bound on the
  expected time from the conclusion of the theorem.
  A simple way (not necessarily the only way) to evaluate the sum
  $\sum_{j=1}^{m-1}\ln\left( \frac{6\delta\lambda}{4+z_j\delta\lambda}
  \right)$ is to underestimate the denominator $4 + z_j\delta\lambda$ in
  each term by either $4$, or $z_j\delta\lambda$. This gives the bounds
  $m\ln(3\lambda/2)$, or $\sum_{j=1}^{m-1}\ln(6/z_j)$ for this sum.
\end{enumerate}
Some iterations of the above steps may be required to find parameter
settings that yield the best possible bound.

We now illustrate this methodology on a simple example.

\begin{corollary}\label{cor:toy}
  For any $\lambda\geq 72(\ln(n)+9)$, the expected
    number of points created
  until the population of Algorithm~\ref{algo:toy} contains
  the   point $n$ is $\bigO{n\lambda}$.
\end{corollary}

\begin{algorithm}[H]
\begin{algorithmic}[1]
  \REQUIRE Finite state space $\mathcal{X}=\{1,\ldots, n\}$ for some $n\in\mathbb{N}$.
  \STATE $P_0\sim \unif(\mathcal{X}^\lambda)$, i.e. initial population
  sampled u.a.r.
  \FOR{$t=0, 1, 2,\ldots $ until termination condition met}
  \FOR{$i=1$ to $\lambda$}
  \STATE Sort the current population $P_t=(x_1,\ldots, x_\lambda)$\\
         such that $x_1\geq x_2\geq \cdots \geq x_\lambda$. \label{algo:toy:line:opDstart}
  \STATE $z:=x_k$ where $k\sim \unif(\{1,\ldots, \lambda/2\})$. \label{algo:toy:line:select}
  \STATE $y:=z + \unif(\{-c,0,1\})$ for any $c \in [n]$ \label{algo:toy:line:opDend}
  \STATE $P_{t+1}(i) := \max\{1,\min\{y,n\}\}$.
  \ENDFOR
  \ENDFOR
\end{algorithmic}
\caption{Example algorithm to illustrate Theorem~\ref{thm:general-fitness-levels}.}
\label{algo:toy}
\end{algorithm}

The purpose of Algorithm~\ref{algo:toy} is to illustrate the
application of Theorem~\ref{thm:general-fitness-levels} on a very
simple example. The search space $\mathcal{X}$ is the set of natural
numbers between 1 and $n$. A population is a vector of $\lambda$ such
numbers, and the implicit objective is to obtain a population
containing the number $n$. Following the scheme of
Algorithm~\ref{algo:Algorithm1}, the operator $D$ corresponds to
lines~\ref{algo:toy:line:opDstart}-\ref{algo:toy:line:opDend}.
The new individual $y$ is obtained by first selecting uniformly
at random one of the best $\lambda/2$ individuals in the population
(lines~\ref{algo:toy:line:opDstart} and \ref{algo:toy:line:select}),
and ``mutating'' this individual by adding $1$, subtracting $c$, or
do nothing, with equal probabilities.
The value of $c$ does not matter in our analysis. Note that for $c$ being fixed to $1$, one could ignore the selection
steps and easily come up with a rough bound $\bigO{n^2\lambda}$. However, for other choices of $c$, \eg equal to $n$ or randomly picked,
without the tool proposed in Theorem~\ref{thm:general-fitness-levels}
it is much less obvious how such a process should be approached and
analysed.

We now carry out the steps described previously.

Step 1: It seems natural to
partition the search space into $m=n$ levels, where $A_j:=\{j\}$ for
all $j\in[m]$.

Step 2: Assume that the \emph{current level} is $j<n-1$. This
means that in $P_t$, there are $\gamma_0\lambda$ individuals in
$A_{\geq j}$, i.e. with fitness at least $j$, and at least
$\gamma\lambda$ but less than $\gamma_0\lambda$ individuals in
$A_{\geq j+1}$, i.e. with with fitness at least $j+1$. We need to
estimate $\Pr_{y\sim D(P_t)}(y\in A_{\geq j+1})$, i.e., the
probability of producing an individual with fitness at least
$j+1$. To this end, we say that a selection event is ``good'' if
in step 5, the algorithm selects an individual in $A_{\geq j+1}$,
i.e. with fitness at least $j+1$. If $\gamma\leq 1/2$, then the
probability of a good selection event is at least
$\gamma\lambda/(\lambda/2)=2\gamma$. And we say that a mutation
event is ``good'' if in line~\ref{algo:toy:line:select}, the
algorithm does not subtract $1$ from the selected search point.
The probability of a good mutation event is $2/3$. Selection and
mutation are independent events, hence we have shown for all
$\gamma\in(0,1/2]$ that
\begin{align*}
  \Pr_{y\sim D(P_t)}(y\in A_{\geq j+1}) \geq (2\gamma)(2/3) = \gamma\left(1+\frac{1}{3}\right).
\end{align*}
Condition (G2) is therefore satisfied with $\delta=1/3$ if we choose
any positive constant $\gamma_0 \leq 1/2$.

Step 3: Assume that population $P_t$ has at least
$\gamma_0\lambda$ individuals in $A_{\geq j}$. In this case, the
algorithm produces an individual in $A_{\geq j+1}$ if in
line~\ref{algo:toy:line:select} it selects an individual in
$A_{\geq j}$ and mutates the individual by adding 1 in
line~\ref{algo:toy:line:opDend}. If we now fix $\gamma_0 = 1/2$,
the probability of selecting an individual in $A_{\geq j}$ is $1$
by the assumption. Furthermore, the probability of adding $1$ to
the selected individual is exactly $1/3$. Hence, we have shown
\begin{align*}
  \Pr_{y\sim D(P_t)}(y\in A_{\geq j+1}) \geq 1(1/3),
\end{align*}
and we can satisfy condition (G1) by defining $z_j:=1/3$ for all $j\in[m-1]$.

Step 4: For the parameters we have chosen, it is easy to see by
numerical calculation that the population size
$\lambda\geq 72(\ln(n)+9)$ satisfies condition (G3).

Step 5: We use $\sum_{j=1}^{m-1}\ln\left(\frac{6}{z_j}
\right)$ instead of $\sum_{j=1}^{m-1}\ln\left(\frac{6\delta\lambda}{4
+ \delta \lambda z_j} \right)$, thus the expected time until the
population has found the point $n$ is no more than
$$\bigO{
    \sum_{j=1}^{n-1}\frac{1}{3}
    + \lambda\sum_{j=1}^{n-1}\ln\left(\frac{6}{1/3}\right)}
  = \bigO{n\lambda}.
$$

\subsection{Proof of the level-based theorem}

Theorem~\ref{thm:general-fitness-levels} will be proved using drift
analysis, which is a standard tool in theory of randomised search
heuristics. Our distance function takes into account both the
``current level'' of the population, as well as the distribution of
the population around the current level.
In particular, let the current level $Y_t$ be the highest level
$j\in[m]$ such that there are at least $\gamma_0\lambda$
individuals at level $j$ or higher. Furthermore, for any level
$j\in[m]$, let $X^{(j)}_{t}$ be the number of individuals at level $j$ or
higher. Hence, we describe the dynamics of the population by $m+1$
stochastic processes $X^{(1)}_t,\ldots,X^{(m)}_{t},Y_t$.  Assuming
that these processes are adapted to a filtration $\filt{t}$, we write
$\expectt{X}:=\expect{X\mid\filt{t}}$ and
$\Pr_t(\mathcal{E}):=\expect{\indfe{\mathcal{E}}\mid\filt{t}}$.
Our approach is to measure the
distance of the population at time $t$ by a scalar
$g(X^{(Y_t+1)}_{t}, Y_t)$, where $g$ is a function that satisfies the
conditions in Definition~\ref{def:property}.

\begin{definition}\label{def:property}
  A function $g:\{0\}\cup[\lambda]\times [m]\rightarrow\mathbb{R}$ is called
  a \emph{\fprop} if
  \begin{enumerate}
  \item $\forall x\in\{0\}\cup[\lambda], \forall y\in[m-1]\quad
         g(x,y) \geq g(x,y+1)$,
  \item $\forall x\in\{0\}\cup[\lambda-1], \forall y\in[m]\quad
         g(x,y)\geq g(x+1,y)$, and
  \item $\forall y\in[m-1]\quad
         g(\lambda,y)\geq g(0,y+1)$.
  \end{enumerate}
\end{definition}
It is clear from the definition that the sum of two level functions
is also a level function. In addition, the three
conditions ensure that the distance $g(X^{(Y_t+1)}_{t}, Y_t)$ of
the population decreases monotonically with the current level $Y_t$.
As the following lemma shows, this monotonicity allows an upper bound
on the distance in the next generation which is partly independent
of the change in current level.
\begin{lemma}\label{lemma:increase}
  If $Y_{t+1}\geq Y_t,$ then for any \fprop $g$
  \begin{align*}
         g\left(X^{(Y_{t+1}+1)}_{t+1},Y_{t+1}\right)
    \leq g\left(X^{(Y_t+1)}_{t+1},Y_t\right).
  \end{align*}
\end{lemma}
\begin{proof}
  The statement is trivially true when $Y_t=Y_{t+1}$. On the other hand,
  if $Y_{t+1}>Y_t$, then the conditions in Definition~\ref{def:property} imply
  \begin{align*}
      g\left(X^{(Y_{t+1}+1)}_{t+1},Y_{t+1}\right)
      & \leq g\left(0,Y_{t+1}\right)
        \leq g\left(0,Y_{t}+1\right)\\
      & \leq g\left(\lambda,Y_{t}\right)
        \leq g\left(X^{(Y_t+1)}_{t+1},Y_t\right). \qedhere
  \end{align*}
\end{proof}

We can now give the formal proof of Theorem~\ref{thm:general-fitness-levels}.

\begin{proof}[Proof of Theorem \ref{thm:general-fitness-levels}]
We will prove the theorem using Lemma~\ref{lemma:pol-drift} (the
additive drift theorem) with respect to the parameter $a=0$ and a
stochastic process
\begin{align*}
  Z_t := g\left(X_{t}^{(Y_t+1)},Y_t\right),
\end{align*}
where $g$ is a level-function to be defined, and
$(Y_t)_{t\in\mathbb{N}}$ and $(X^{(j)}_t)_{t\in\mathbb{N}}$ for
$j\in[m]$ are stochastic processes to be defined.
We consider the filtration $(\mathscr{F}_t)_{t\in\mathbb{N}}$ induced by
the sequence of populations $(P_t)_{t\in\mathbb{N}}$.

We will assume w.l.o.g. that condition~(G2) is also satisfied
for $j=m-1$, for the following reason.
Given an Algorithm~1 with certain mapping~$D$, consider
Algorithm~1 with a different mapping~$D'(P)$:
If~$|P\cap A_{m}|=0$ then $D'(P)=D(P)$; otherwise $D'(P)$
assigns probability mass $1$ to some element $x$ of~$P$ that is in
$A_{m}$, e.g. to the first one among such elements.
Note that $D'$ meets conditions~(G1) and (G2). Moreover, (G2)
holds for $j=m-1$.
For the sequence of populations $P'_0,P'_1,\dots$ of Algorithm~1
with mapping~$D'$ we can put ${T' := \lambda \cdot \min\{t \mid
|P'_t\cap A_{m}| > 0\}}$. Executions of the original
algorithm and the modified one before generation $T'/\lambda$ are
identical. On generation~$T'/\lambda$ both algorithms place
elements of~$A_{m}$ into the population for the first time.
Thus, $T'$  and $T$ are equal in every realisation and their expected
values is the same.

  For any level $j\in[m]$ and time $t\geq 0$, let the random variable
  \begin{align*}
    X_t^{(j)} := | P_t\cap A_{\geq j} |
  \end{align*}
  denote the number of individuals in levels $A_{\ge j}$ at time $t$.
  Because $A_{\geq j}$ is partitioned into disjoint sets $A_{j}$ and
  $A_{\geq j+1}$, the definition implies
  \begin{align}
    |P_t \cap A_{j}| = X^{(j)}_{t} - X^{(j+1)}_{t} \label{eq:indv-at-level-j}
  \end{align}
  Algorithm~\ref{algo:Algorithm1} samples all individuals in generation
  $t+1$ independently from distribution $D(P_{t})$. Therefore, given the
  current population $P_t$,   $X^{(j)}_{t+1}$ is binomially distributed
  \begin{align*}
    X_{t+1}^{(j)} \sim\bin(\lambda, p_{t+1}^{(j)})
  \end{align*}
  where
  $
    p_{t+1}^{(j)} := \Pr_{t,y\sim D(P_{t})}\left( y\in A_{\ge j}\right)
  $
  is the probability of sampling an individual in level $j$ or higher.

  The \emph{current level} $Y_t$ of the population at time $t\geq 0$
  is defined as
  \begin{align*}
    Y_t := \max \left\{ j\in[m] \;\middle|\; X_t^{(j)} \geq  \gamma_0\lambda \right\}.
  \end{align*}
  Note that  $(X^{(j)}_t)_{t\in\mathbb{N}}$ and
  $(Y_t)_{t\in\mathbb{N}}$ are adapted to the filtration
  $(\mathscr{F}_t)_{t\in\mathbb{N}}$ because they are defined in terms
  of the population process $(P_t)_{t\in\mathbb{N}}$.

  When $Y_t <m$, there exists a unique $\gamma<\gamma_0$ such that
  \begin{align}
    X_t^{(Y_t+1)} & = |P_t\cap A_{\geq Y_t+1}| =  \gamma \lambda\label{eq:config-3}\\
    X_t^{(Y_t)}   & = |P_t\cap A_{\geq Y_t}| \geq \gamma_0\lambda\label{eq:config-2}\\
    X_t^{(Y_t-1)}  & = |P_t\cap A_{\geq Y_t-1}| \geq \gamma_0\lambda\label{eq:config-1}
  \end{align}

  In the case of $X_t^{(Y_t+1)} = 0$, it follows from
  \eqref{eq:indv-at-level-j}, \eqref{eq:config-3} and \eqref{eq:config-2} that $|P
  \cap A_j| = X^{(Y_t)}_t \geq \gamma_0 \lambda$. Condition (G1) for level $j=Y_t$
  then gives
  \begin{align}
    p_{t+1}^{(Y_t+1)} = \Pr_{y\sim D(P_t)}(y\in A_{\geq Y_t+1})
      \geq z_{Y_t} \label{eq:prob-plusindv-when-xnull}
  \end{align}
  Otherwise if $X_t^{(Y_t+1)} \geq 1$, conditions (G1) and (G2) for level
  $j=Y_t$ with \eqref{eq:config-3} and \eqref{eq:config-2} imply
  \begin{align}
    p_{t+1}^{(Y_t+1)} & = \Pr_{y\sim D(P_t)}(y\in A_{\geq Y_t+1})\\
      & \geq \max\left\{(1+\delta)\frac{X_t^{(Y_t+1)}}{\lambda}, z_j\right\} \label{eq:prob-plusindv-when-xpos}.
  \end{align}
  Condition (G2) for level $j=Y_t-1$ along with \eqref{eq:config-2} and
  \eqref{eq:config-1} also gives
  \begin{align}
    p_{t+1}^{(Y_t)} & = \Pr_{y\sim D(P_t)}(y\in A_{\geq Y_t})
                   \geq (1+\delta)\gamma_0.\label{eq:prob-curindv}
  \end{align}

  We now define the process $(Z_t)_{t\in\mathbb{N}}$ as
  $Z_t:=0$ if $Y_t=m$, and otherwise, if $Y_t<m$, we let
  $
    Z_t:=g\left(X_{t}^{(Y_t+1)},Y_t\right),
  $
  where for all $k$, and for all $1\leq j<m$,
  $g(k,j)=g_1(k,j)+g_2(k,j)$ and
  \begin{align*}
    g_1(k,j)
      &:= \ln\left(\frac{1 + (\delta/2)\lambda}{1 + (\delta/2)\max\{k,z_j\lambda/(1+\delta)\}}\right) \\
      &\quad\quad + \sum_{i=j+1}^{m-1} \ln\left(\frac{1 + (\delta/2)\lambda}{1+(\delta/2)\lambda z_i/(1+\delta)}\right)\\
    g_2(k,j)
      &:= \displaystyle \frac{1}{q_{j} }\left(1 - \frac{\delta^{2}}{7}\right)^k + \sum^{m-1}_{i=j+1} \frac{1}{q_i}
  \end{align*}
  where $q_j:=1-(1-z_j)^\lambda$.

   It follows from Lemma~\ref{lemma:level-functions} that $g(k,j)$ is a level function.
   Furthermore, $g(k,j) \geq 0$ for all $k \in \{0\}\cup[\lambda]$ and all $j \in [m]$.
     Due to properties 1 and 2 of \fprop{s} and Lemma~\ref{lemma:exp-ineq},
     the distance is always bounded from above by
     \begin{align}
       g(0,1)
       & \leq  \sum_{i=1}^{m-1}\left(\ln\left(\frac{1+(\delta/2)\lambda}{1+(\delta/2)z_i\lambda/(1+\delta)}\right) + \frac{1}{q_i}\right)\nonumber \\
       & <     \sum_{i=1}^{m-1}\left(\ln\left(\frac{4+2\delta\lambda}{4+\delta z_i\lambda}\right) +1 +\frac{1}{\lambda z_i}\right)\label{eq:distance-bound-1}
       \intertext{using that $z_i \leq 1$,}
       & <     \sum_{i=1}^{m-1}\left(\ln\left(\frac{4+2\delta\lambda}{z_i(4+\delta \lambda)}\right) +1 +\frac{1}{\lambda z_i}\right)\\ \nonumber
       & =     \sum_{i=1}^{m-1}\left(\ln\left(\frac{1}{z_i}\left(1+\frac{\delta\lambda}{4+\delta\lambda}\right)\right) +1 +\frac{1}{\lambda z_i}\right) \nonumber
       \intertext{$\ln(x)\leq x-1$ for all $x>0$ so}
       & <     \sum_{i=1}^{m-1}\left(\frac{2}{z_i} +\frac{1}{\lambda z_i}\right)\nonumber \\
       \intertext{and $z_i\geq z_*$ and $\lambda \geq \lceil 4\ln(128)\rceil= 20$ from (G3) so}
       & <  \frac{m}{z_*}\left(2+\frac{1}{\lambda}\right)           \leq  \frac{41m}{20z_*}\label{eq:distance-bound-2}.
     \end{align}
     Hence, we have $0\leq
     Z_t<g(0,1)<\infty$ for all $t\in\mathbb{N}$ which implies that
     $\expect{Z_t}<\infty$ for all $t\in\mathbb{N}$,
     and condition 2 of the drift theorem is satisfied.

  The ``drift'' of the process is the random variable
  \begin{align*}
    \Delta_{t+1}  & := g\left(X_t^{(Y_t+1)},Y_t\right)-g\left(X_{t+1}^{(Y_{t+1}+1)},Y_{t+1}\right).
  \end{align*}
  To compute the expected drift, we apply the law of total probability
  \begin{align}
  \expectt{\Delta_{t+1}}
         &  =  (1-\probt{Y_{t+1}<Y_t})\expectt{\Delta_{t+1} \mid Y_{t+1}\geq Y_t} \nonumber \\
         &\quad\; + \probt{Y_{t+1}<Y_t}\expectt{\Delta_{t+1} \mid Y_{t+1}< Y_t}.\label{eq:law-tot-prob}
  \end{align}
  The event $Y_{t+1}< Y_t$ holds if and only if
  $X_{t+1}^{(Y_t)}<\gamma_0\lambda$. Due to \eqref{eq:prob-curindv}, we obtain
  the following by a Chernoff bound
  \begin{align}
    \probt{Y_{t+1}<Y_t}
      &=    \probt{X_{t+1}^{(Y_t)}<\left(1-\frac{\delta}{1+\delta}\right)(1+\delta)\gamma_0\lambda} \nonumber\\
      &\leq
            \exp\left(-\frac{\delta^2\gamma_0\lambda}{2(1+\delta)}
            \right)
       \leq \frac{\delta^2z_*}{128m}, \label{eq:prob-fall}
  \end{align}
  where the second last inequality takes into account the population size
  required by condition (G3).
  Given the low probability of the event $Y_{t+1}<Y_t$, it suffices
  to use the pessimistic bound from (\ref{eq:distance-bound-2})
  \begin{align}
    \expectt{\Delta_{t+1} \mid Y_{t+1}<Y_t}
       &\geq -g(0,1)\geq - \frac{41m}{20z_*}. \label{eq:drift-fall}
  \end{align}

  If $Y_{t+1}\geq Y_t$, we can apply Lemma~\ref{lemma:increase}
  \begin{align*}
    &\expectt{\Delta_{t+1} \mid Y_{t+1}\geq Y_t} \\
      &\quad\geq
      \expectt{g\left(X^{(Y_t+1)}_{t},Y_t\right) - g\left(X^{(Y_t+1)}_{t+1},Y_{t}\right)
               \mid Y_{t+1}\geq Y_t}.
  \end{align*}

  Note that event $Y_{t+1}\geq Y_t$ is equivalent to having $X^{(Y_t)}_{t+1} \geq
  \gamma_0\lambda$, then due to Lemma~\ref{appendix:lemma:skip-condition-in-prob}, in
  the following we can skip the condition on the event when needed.

  If $X_{t}^{(Y_t+1)}=0$, then $X_{t}^{(Y_t+1)}\leq X_{t+1}^{(Y_t+1)}$ and
  \begin{align*}
    \expectt{g_1\left(X_t^{(Y_t+1)},Y_t\right) - g_1\left(X_{t+1}^{(Y_t+1)},Y_t\right)
             \mid Y_{t+1}\geq Y_t}\geq 0,
  \end{align*}
  because the function $g_1$ satisfies property~2 in Definition~\ref{def:property}.
  Furthermore, we have the lower bound
                \begin{multline*}
    \expectt{g_2\left(X_t^{(Y_t+1)},Y_t\right)-g_2\left(X_{t+1}^{(Y_t+1)},Y_t\right)
             \mid Y_{t+1}\geq Y_t} \\
      >    \probt{X_{t+1}^{(Y_t+1)}\geq 1}
           \left(g_2\left(0,Y_t\right)-g_2\left(1,Y_t\right)\right)
      \geq \frac{\delta^{2}}{7},
  \end{multline*}
  where the last inequality follows because of
  \eqref{eq:prob-plusindv-when-xnull} and
  $
    \probt{X_{t+1}^{(Y_t+1)}\geq 1}
      \geq 1-\left(1-p^{(Y_t+1)}_{t+1}\right)^\lambda
      \geq 1-(1-z_{Y_t})^\lambda = q_{Y_t}$, and
  $ g_2\left(0,Y_t\right)-g_2\left(1,Y_t\right)
      = (1/q_{Y_t})(\delta^2/7)$.

  In the other case, where $X_t^{(Y_t+1)}\geq 1$, we obtain
  \begin{align*}
    &\expectt{g_1\left(X_t^{(Y_t+1)},Y_t\right)-g_1\left(X_{t+1}^{(Y_t+1)},Y_t\right)
              \mid Y_{t+1} \geq Y_{t}} \\
      &\quad\geq \expectt{ \ln\left( \frac{1+\frac{\delta}{2} X^{(Y_t+1)}_{t+1}}{1+ \frac{\delta}{2}\max\left\{
      X^{(Y_t+1)}_t,\frac{z_{Y_t}\lambda}{1+\delta}\right\}}\right)}
      \geq \frac{\delta^{2}}{7}.
  \end{align*}
  where the last inequality follows from \eqref{eq:prob-plusindv-when-xpos} and
  Lemma~\ref{appendix:lemma:ln-expect-bound}.
  For function $g_2$, we get
  \begin{multline*}
    \expectt{g_2\left(X_t^{(Y_t+1)},Y_t\right)-g_2\left(X_{t+1}^{(Y_t+1)},Y_t\right)
             \mid Y_{t+1} \geq Y_{t}}
      = \\
    \frac{1}{q_{Y_t}}
    \left(
      \left(1 - \frac{\delta^2}{7}\right)^{X_t^{(Y_t)}}
              - \expectt{\left(1 - \frac{\delta^2}{7}\right)^{X_{t+1}^{(Y_t+1)}}}
      \right)>0
  \end{multline*}
  where the last inequality is due to
  Lemma~\ref{appendix:lemma:neg-moment-binomial}, applied to
  \[X_{t+1}^{(Y_t+1)}\sim\bin(\lambda,p_{t+1}^{(Y_t+1)})\] with
  $p_{t+1}^{(Y_t+1)}\geq (1+\delta)X_t^{(Y_t+1)}/\lambda$ (see
  \eqref{eq:prob-plusindv-when-xpos}) and the parameter
  \[\kappa=-\ln(1 - \delta^{2}/7)<\delta.\]

  Taking into account all cases, we have
  \begin{align}
    \expectt{\Delta_{t+1} \mid Y_{t+1}\geq Y_t }
      \geq \frac{\delta^{2}}{7}.\label{eq:drift-forward}
  \end{align}

  We now have bounds for all the quantities in~\eqref{eq:law-tot-prob}
  with \eqref{eq:prob-fall}, \eqref{eq:drift-fall} and \eqref{eq:drift-forward},
  and we get
  \begin{align*}
  \expectt{\Delta_{t+1}}
    &=       (1 - \probt{Y_{t+1}<Y_t})\expectt{\Delta_{t+1} \mid Y_{t+1}\geq Y_t} \\
    &\quad\;    + \probt{Y_{t+1}<Y_t}\expectt{\Delta_{t+1} \mid Y_{t+1}< Y_t} \\
    &\geq    \left(1-\frac{\delta^2z_*}{128m}\right)\frac{\delta^{2}}{7}
           - \left(\frac{\delta^2z_*}{128m}\right)\left(\frac{41m}{20z_*}\right)
   > \frac{\delta^2}{8}.
  \end{align*}

  We now verify condition 3 of the drift theorem
  (Lemma~\ref{lemma:pol-drift}), i.e., that $T$ has
  finite expectation. Let $p_*:=\min\{(1+\delta)/\lambda, z_*\}$, and
  note by conditions (G1) and (G2) that the current level increases by at
  least one with probability
  $\probt{Y_{t+1}>Y_t}\geq (p_*)^{\gamma_0\lambda}$. Furthermore,
  due to the definition of the modified process $D'$, if $Y_t=m$ then $Y_{t+1}=m$.
  Hence, the probability of reaching $Y_t=m$ is lower bounded by the
  probability of the event that the current level increases in all of
  at most $m$ consecutive
  generations, i.e., $\probt{Y_{t+m} =m} \geq (p_*)^{\gamma_0\lambda m}>0$.
  It follows that $\expect{T}<\infty$.

  By Lemma~\ref{lemma:pol-drift} and the upper bound on $g(0,1)$ in \eqref{eq:distance-bound-1},
  \begin{align*}
    \expect{T}
      &\leq \lambda\cdot \frac{g(0,1)}{\delta^{2}/8}\\
      &<  \left(\frac{8\lambda}{\delta^{2}}\right)
            \sum_{i=1}^{m-1}\left(\ln\left(\frac{4+2\delta\lambda}{4+z_i\delta\lambda }\right)
        +1+\frac{1}{\lambda z_i}\right)
        \intertext{then using that $4\leq \delta\lambda/5$ from (G3) and $(1/5+2)e<6$}
      &< \left(\frac{8\lambda}{\delta^{2}}\right)
             \sum_{i=1}^{m-1}\left(\ln\left(\frac{6\delta\lambda}{4+z_i\delta\lambda}\right)+
                           \frac{1}{\lambda z_i }\right).
                  \qedhere
  \end{align*}
\end{proof}

\section{Tools for Analysis of Genetic Algorithms}\label{sec:cor}

This section provides two corollaries of Theorem~\ref{thm:general-fitness-levels}
tailored to Algorithm~\ref{algo:GA}.
After that, we give sufficient conditions for tunable parameters of
many selection mechanisms allowing the applicability of the corollaries.

Since no explicit fitness function is defined in Algorithm~\ref{algo:GA},
no assumption on a $f$-based partition will be required by the corollaries.
Nevertheless, we have to generalise the \emph{cumulative selection probability}
function of $\sel$ operator, denoted $\beta(\gamma,P)$ \cite{bib:dl16} which
is defined relative to the fitness function $f$, to the one that is relative
to the order of the partition $(A_1,\dots,A_m)$.

Recall that to define $\beta(\gamma,P)$ of $\sel$ \wrt $f$ for a population $P$
of $\lambda$ search points, we first assume
$(f_{1},\dots,f_{\lambda})$ to be the vector of sorted fitness values of $P$,
\ie $f_{i} \geq f_{i+1}$ for each $i \in [\lambda-1]$.
Then given $\gamma \in (0,1]$, define
\begin{align*}
  \beta(\gamma,P)
    := \sum_{i = 1}^{\lambda} \psel(i \mid P)
         \cdot \left[ f(P(i))
                                            \geq f_{\lceil\gamma\lambda\rceil} \right],
\end{align*}
where $[\cdot]$ denotes the Iverson bracket.

Similarly, given a partition $(A_1,\dots,A_m)$, if we use $(\ell_{1},\dots,
\ell_{\lambda})$ to denote the sorted levels of search points in $P$, \ie
$\ell_{i} \geq \ell_{i+1}$ for each $i \in [\lambda-1]$, then
the \emph{cumulative selection probability} of function of $\sel$ \wrt
$(A_1,\dots,A_m)$ is \begin{align*}
  \genbeta(\gamma,P)
    := \sum_{i = 1}^{\lambda} \psel(i \mid P)
         \cdot
           \left[P(i) \in A_{\geq \ell_{\lceil\gamma\lambda\rceil}} \right].
         \end{align*}

Let $(A_1,\dots,A_m)$ be a $f$-based partition of $\mathcal{X}$, it follows
from the above definitions that
\begin{align}
  \forall P \in \mathcal{X}^\lambda,
  \forall \gamma \in (0,1]
    \quad \genbeta(\gamma,P) \geq \beta(\gamma,P)\label{eq:relation-genbeta-beta}
\end{align}
for which the equality occurs when the partition is canonical.

\subsection{Analysis of non-permanent use of crossover}

We first derive from Theorem~\ref{thm:general-fitness-levels} a corollary
that is adapted to Algorithm~\ref{algo:GA} with $p_\mathrm{c}<1$.
This setting covers the case $p_\mathrm{c}=0$, \ie only unary variation
operators are used. This specific case is the main subject of~\cite{bib:dl16},
and to some extent our corollary shares many similarities with the main theorem
of that paper. As we will see later on, stronger and more general results can
be claimed with the corollary.

\begin{corollary}\label{cor:GA}
Let $(A_1,\dots,$ $A_{m})$ be a partition of $\mathcal{X}$, define
$T:=\min\{t\lambda \mid |P \cap A_{m}| > 0\}$ to be the first
point in time that Algorithm~\ref{algo:GA} with $p_\mathrm{c}<1$
obtains an element of $A_{m}$.
If there exist
  $s_1,\dots,s_{m-1},s_*,p_0,\delta \in(0,1]$,
and a constant $\gamma_0 \in (0,1)$ such that
  \begin{description}[noitemsep]
  \item[(M1)] for each level $j\in [m-1]$
    \[\displaystyle
       \pmut\left( y\in A_{\geq j+1} \mid x \in A_j \right)\geq s_j\]
  \item[(M2)] for each level $j\in [m-1]$
   \[\displaystyle
     \pmut\left( y\in A_{\geq j} \mid x\in A_j \right)\geq p_0\]
                \item[(M3)] for any population $P \in \left(\mathcal{X}\setminus A_{m}\right)^\lambda$
              and $\gamma\in(0,\gamma_0]$
  \[\displaystyle
               \genbeta(\gamma, P) \geq \frac{(1+\delta)\gamma}{p_0(1-p_\mathrm{c})}\]
  \item[(M4)] the population size $\lambda$ satisfies
  \[\displaystyle \lambda \geq
               \left(\frac{4}{\gamma_0\delta^2}\right)\ln\left(\frac{128 m}{\gamma_0s_*\delta^2}\right)
               \text{ where } s_*:=\min_{j\in[m-1]} \{s_j\},
  \]
  \end{description}
  then
  \begin{align*}
  \expect{T}
   <
    \left(\frac{8}{\delta^{2}}\right)
    \sum_{j=1}^{m-1}\left(\lambda
    \ln\left(\frac{6\delta\lambda}{4+\gamma_0
    s_j\delta\lambda}\right)+\frac{1}{\gamma_0 s_j}\right).
  \end{align*}
\end{corollary}
\begin{proof}
Following the guideline, we apply Theorem~\ref{thm:general-fitness-levels}.
Step~1 is skipped because we already have the partition.
Step 2: We assume $|P \cap A_{\geq j}| \geq \gamma_0 \lambda$ and
$|P \cap A_{\geq j+1}|\geq \gamma\lambda > 0$ for some $\gamma
\leq \gamma_0$. Hence, to create an individual in $A_{\geq j+1}$
it suffices to pick an $x \in |P \cap A_k|$ for any $k \geq j+1$
and mutate it to an individual in $A_{\geq k}$, the probability of
such an event according to (M2) and (M3) is at least $(1 -
p_\mathrm{c}) \genbeta(\gamma,P) p_0  \geq (1+\delta)\gamma$. So
(G2) holds for the same $p_0$, $\delta$ and $\gamma_0$ as in (M3).

Step 3: We are given $|P \cap A_{j}| \geq \gamma_0 \lambda$.
Thus, with probability $\genbeta(\gamma_0,P)$, the selection
mechanism chooses an individual $x$ in either $A_j$ or $A_{\geq
j+1}$. If the individual $x$ belongs to $A_j$, then the mutation
operator will by condition (M1) upgrade the individual to $A_{\geq
j+1}$ with probability $s_j$. If the individual belongs to
$A_{\geq j+1}$, then by (M2), the mutation operator maintains the
individual in $A_{\geq j+1}$ with probability $p_0$. Finally, no
crossover occurs with probability $1-p_c$. Hence, the probability
of producing an individual in $A_{\geq j+1}$ is at least
\begin{align*}
 (1-p_\mathrm{c})\genbeta(\gamma_0,P) \min\{s_j,p_0\}
  & \geq (1-p_\mathrm{c})\genbeta(\gamma_0,P) s_j p_0\\
  & > \gamma_0 s_j =: z_j.
\end{align*}
Thus (G1) holds for that choice of $z_j$ and $z_*:=\gamma_0 s_*$.

Step 4: Given our choice of $z_*$, we have
that condition (M4) implies condition (G3).

For the last step, all conditions (G1-3) are satisfied, and
Theorem \ref{thm:general-fitness-levels} gives
\begin{align*}
  \expect{T}
  & \leq
    \frac{8\lambda}{\delta^{2}}
    \sum_{j=1}^{m-1}\left(
    \ln\left(\frac{6\delta\lambda}{4+z_j\delta\lambda}\right)+\frac{1}{z_j\lambda}\right)\\
  & =
    \frac{8}{\delta^{2}}
    \sum_{j=1}^{m-1}\left(
    \lambda\ln\left(\frac{6\delta\lambda}{4+\gamma_0
    s_j\delta\lambda}\right)+\frac{1}{\gamma_0 s_j}\right).
    \qedhere
\end{align*}
\end{proof}

From the proof, we remark that any operator can be used in place of crossover
in line~\ref{algo:GA:crossover} of Algorithm~\ref{algo:GA}, and the result
still holds. Therefore, the corollary is in fact applicable to a wider range
of algorithms than just Algorithm~\ref{algo:GA}.

\subsection{Analysis of permanent use of crossover}

The following corollary adapts Theorem~1 to the setting of Algorithm~\ref{algo:GA}
with $p_\mathrm{c} = 1$.

\begin{corollary}\label{cor:GA1}
  Given a partition $(A_1,\dots,$ $A_{m})$ of $\mathcal{X}$, let
  $T:=\min\{t\lambda \mid |P \cap A_{m}| > 0\}$ be the first point in time
  that Algorithm~\ref{algo:GA} with $p_\mathrm{c}=1$ obtains an element of
  $A_{m}$.
    If there exist
  $s_1,\dots,s_{m-1},s_*,p_0,\varepsilon,\delta \in(0,1]$,
  and a constant $\gamma_0 \in (0,1)$
such that
      \begin{description}[noitemsep]
  \item[(C1)] for each level $j \in [m-1]$
               \[\displaystyle
               \pmut\left( y\in A_{\geq j+1} \mid x \in A_{j} \right) \geq s_j\]
    \item[(C2)] for each level $j \in [m]$
              \[\displaystyle
               \pmut\left( y\in A_{\geq j} \mid x\in A_{j} \right)\geq p_0\]
    \item[(C3)] for each level $j \in [m-2]$
              \[\displaystyle \quad
               \pxor\left( x\in A_{\geq j+1} \mid u \in A_{\geq j}, v \in A_{\geq j+1}\right) \geq \varepsilon\]
    \item[(C4)] for any population $P \in \left(\mathcal{X}\setminus A_{m}\right)^\lambda$
              and $\gamma\in(0,\gamma_0]$
              \[\displaystyle
               \genbeta(\gamma, P) \geq \gamma\sqrt{\frac{1+\delta}{p_0\gamma_0\varepsilon}}\]
  \item[(C5)] the population size satisfies
              \[\displaystyle \lambda \geq \left(\frac{4}{\gamma_0\delta^2}\right)\ln\left(\frac{128 m }{\gamma_0\delta^2 s_*}\right)
              \text{ where } s_*:=\min_{j\in[m-1]} \{s_j\}\]
  \end{description}
  then
  \begin{align*}
  \expect{T}
   <
    \left(\frac{8}{\delta^{2}}\right)
    \sum_{j=1}^{m-1}\left(\lambda
    \ln\left(\frac{6\delta\lambda}{4+\gamma_0
    s_j\delta\lambda}\right)+\frac{1}{\gamma_0 s_j}\right).
  \end{align*}
\end{corollary}

\begin{proof}
We apply Theorem~\ref{thm:general-fitness-levels} following the guideline.
Again, Step~1 is skipped because the partition is already defined.

Step 2: We are given $|P \cap A_{\geq j}| \geq \gamma_0 \lambda$
and $|P \cap A_{\geq j+1}| \geq \gamma \lambda > 0$. To create an
individual in $A_{\geq j+1}$, it suffices to pick the individual
$u$ in $A_{\geq j}$ and $v$ in $A_{\geq j+1}$, then to produce an
individual in $A_{k}$ for any $k \geq j+1$ by crossover and not
destroy the produced individual by mutation. The probability of
such an event according to (C2), (C3) and (C4) is bounded from
below by $\genbeta(\gamma_0,P)\genbeta(\gamma,P)\varepsilon p_0
\geq (1 + \delta)\gamma$. Condition (G2) is then satisfied with
the same $\gamma_0$ and $\delta$ as in (C4).

Step 3: We assume $|P \cap A_{j}| \geq \gamma_0 \lambda$.
Note that condition (C3) written for level $j-1$ is $\pxor(x \in A_{\geq j}
\mid u \in A_{\geq j-1}, v \in A_{\geq j}) \geq \varepsilon$, and because
$A_{\geq j} \subset A_{\geq j-1}$ then $\pxor(x \in A_{\geq j} \mid u \in
A_{\geq j}, v \in A_{\geq j}) \geq \varepsilon$. To create an individual in
$A_{\geq j+1}$, it then suffices to pick both $u$ and $v$ from $A_{\geq j}$
in line~\ref{algo:GA:start-iter},
then to produce an individual in $A_{k}$ for any $k \geq j$ by crossover,
now if $k=j$ we need to improve the produced individual by mutation, \ie relying
on (C1), otherwise if $k>j$ it suffices not to destroy the produced individual
by mutation, \ie relying on (C2). It then follows from (C4) that the probability
of producing an individual in $A_{\geq j+1}$ is at least
\[
  \genbeta(\gamma_0,P)^2 \varepsilon\min\{s_j,p_0\}
    \geq \genbeta(\gamma_0,P)^2 \varepsilon s_j p_0
    > \gamma_0s_j =: z_j.
\]
Condition (G1) then holds for that choice of $z_j$
and $z_*:= \gamma_0 s_*$.

Step 4: It follows from the above definition of $z_*$ that (C5) implies (G3).

In the last step, since all conditions (G1-3) are satisfied,
Theorem~\ref{thm:general-fitness-levels} guarantees that
\begin{align*}
  \expect{T}
  & \leq
            \frac{8}{\delta^{2}}
    \sum_{j=1}^{m-1}\left(
    \lambda\ln\left(\frac{6\delta\lambda}{4+\gamma_0
    s_j\delta\lambda}\right)+\frac{1}{\gamma_0 s_j}\right).
    \qedhere
\end{align*}
\end{proof}

The corollary shares many similarities with Corollary~\ref{cor:GA},
except that condition (C2) has to additionally hold for level $A_m$,
that (C3) is a new condition on the $\xor$ operator, and that condition
(C4) on $\sel$ operator is different from (M3).

\subsection{Analysis of selection mechanisms}

We
show how to parameterise
the following
selection
mechanisms such that condition (M3) of Corollary~\ref{cor:GA}
and (C4) of Corollary~\ref{cor:GA1} are satisfied.
In \emph{$k$-tournament selection}, $k$ individuals are sampled
uniformly at random with replacement from the population, and the
fittest of these individuals is returned.
In $(\mu,\lambda)$-\emph{selection}, parents are sampled uniformly at
random among the fittest $\mu$ individuals in the population.
A function $\alpha:\mathbb{R}\rightarrow\mathbb{R}$ is a ranking
function \cite{bib:Gold89} if $\alpha(x)\geq 0$ for all $x\in[0,1]$,
and $\int_0^1\alpha(x) dx = 1$. In ranking selection with ranking
function $\alpha$, the probability of selecting individuals ranked
$\gamma$ or better is $\int_0^\gamma\alpha(x)dx$.
In \emph{linear ranking selection} parametrised by $\eta \in (1,2]$,
the ranking function is $\alpha(\gamma):=\eta(1 - 2\gamma) + 2\gamma$.
We define \emph{exponential ranking selection}
parametrised by $\eta>0$ with $\alpha(\gamma):=\eta e^{\eta(1 -
\gamma)}/(e^\eta - 1)$.

\begin{lemma}\label{lem:M3}
  Assuming that $(A_1,\dots,A_m)$ is a partition of $\mathcal{X}$
  with $(A_1,\dots,A_{m-1})$ being an $f$-based partition of $\mathcal{X} \setminus A_m$,
  for any constants $\delta'>0,$ $p_0 \in (0,1)$, $\varepsilon \in (0,1)$,
  and for any non-negative parameter $p_\mathrm{c} = 1 - \Omega(1)$,
  there exists a constant $\gamma_0\in (0,1)$
  such that all the following selection mechanisms
  \begin{enumerate}
    \item $k$-tournament selection,
    \item $(\mu,\lambda)$-selection,
    \item linear ranking selection,
    \item exponential ranking selection
  \end{enumerate}
  with their parameters $k$, $\lambda/\mu$ and $\eta$ being set
  to no less than $\displaystyle \frac{1+\delta'}{(1-p_\mathrm{c})p_0}$
  satisfy~(M3),
  \ie $\displaystyle
       \genbeta(\gamma,P)
       \geq \frac{(1+\delta'')\gamma}{p_0(1-p_\mathrm{c})}$
  for any $\gamma \in (0,\gamma_0]$, any ${P\in (\mathcal{X}\backslash
  A_{m})^\lambda}$ and some constant $\delta''>0$.
\end{lemma}
\begin{proof}
Since (M3) only concerns with the restricted subspace $\mathcal{X}\setminus
A_m$ we only need to focus on this subspace, and because the partition is
$f$-based on it, due to \eqref{eq:relation-genbeta-beta} it suffices to prove
the results for $\beta$ function instead of $\genbeta$ function.

The results for $k$-tournament, $(\mu,\lambda)$-selection and
linear ranking follow by applying Lemma~13 in \cite{bib:dl16} (with its $p_0$
being set as our $p_0(1-p_\mathrm{c})$).
For exponential ranking, we first remark the following lower bound,
\begin{align*}
  \beta(\gamma,P)
    &\geq \int_{0}^\gamma \frac{\eta e^{\eta(1 - x)}dx}{e^\eta - 1}
    = \left(\frac{e^\eta}{e^{\eta}-1}\right)\left(1 - \frac{1}{e^{\eta\gamma}}\right)\\
    &\geq 1 - \frac{1}{1 + \eta\gamma}.
\end{align*}
Then the rest of the proof is similar to $k$-tournament with
$\eta$ in place of $k$.
\end{proof}

\begin{lemma}\label{lem:C4}
  Assuming that $(A_1,\dots,A_m)$ is a partition of $\mathcal{X}$
  with $(A_1,\dots,A_{m-1})$ being an $f$-based partition of $\mathcal{X} \setminus A_m$,
  for any constants $\delta'>0,$ $p_0 \in (0,1)$ and
  $\varepsilon \in (0,1)$, there exists a constant $\gamma_0\in(0,1)$
  such that the following selection mechanisms
  \begin{enumerate}
    \item $k$-tournament selection with $k \geq 4(1+\delta')/(\varepsilon p_0)$,
    \item  $(\mu,\lambda)$-selection with $\lambda/\mu \geq (1+\delta')/(\varepsilon p_0)$, and
    \item exponential ranking selection with $\eta \geq 4(1+\delta')/(\varepsilon p_0)$
  \end{enumerate}
  satisfy~(C4),
  \ie $\displaystyle
       \genbeta(\gamma,P)
       \geq \gamma\sqrt{\frac{1+\delta'}{p_0 \varepsilon \gamma_0}}$
  for any $\gamma \in (0,\gamma_0]$ and any ${P\in (\mathcal{X}\backslash
  A_{m})^\lambda}$.
\end{lemma}
\begin{proof}
Similar to the proof of Lemma~\ref{lem:M3}, we only focus on the
subspace $\mathcal{X} \setminus A_m$ where the partition is
$f$-based, and based on \eqref{eq:relation-genbeta-beta} we consider
$\beta$ function instead of $\genbeta$ function.

Define $\varepsilon':=\varepsilon p_0$.

1. Consider $k$-tournament selection and let $\gamma \in
(0,\gamma_0]$.
By the definition of $f$-based partition,
to select an individual from the same
level as the $\gamma$-ranked individual or higher
it is sufficient that the
randomly sampled tournament contains at least one individual with
rank~$\gamma$ or higher. Hence,
\[
\beta(\gamma,P) \ge 1 - (1 - \gamma)^k > 1 - \frac{1}{1+\gamma k},
\]
because $(1 -\gamma)^k < e^{-\gamma k} < \frac{1}{1+\gamma k}$.
So for $k \geq 4(1+\delta')/\varepsilon'$,
\begin{align*}
  \beta(\gamma,P)
    &     \geq 1 - \frac{1}{1+\frac{4\gamma(1+\delta')}{\varepsilon'}}
     = \frac{\frac{4\gamma(1+\delta')}{\varepsilon'}}{\frac{1+4\gamma(1+\delta')}{\varepsilon'}}.
\end{align*}
If $\gamma_0 := \varepsilon'/(4(1+\delta'))$, then for all $\gamma
\in (0, \gamma_0]$ it holds that $4(1+\delta')/\varepsilon' \leq
1/\gamma$ and
\begin{align*}
  \beta(\gamma,P)
    &\geq \frac{\gamma 4(1+\delta')/\varepsilon'}{\gamma (1/\gamma) + 1}
     = \frac{2(1+\delta')\gamma}{\varepsilon'}\\
         &= \sqrt{\frac{(1+\delta')}{\varepsilon' (\varepsilon'/4(1+\delta'))}}
     \gamma
     = \sqrt{\frac{(1+\delta')}{\varepsilon' \gamma_0}} \gamma .
\end{align*}

2. In $(\mu,\lambda)$-selection,
again by $f$-based property of the partition,
we have
$\beta(\gamma,P) = \lambda\gamma/\mu$ if $\gamma \lambda \le \mu$,
and $\beta(\gamma,P) = 1$ otherwise. It suffices to pick $\gamma_0
:= \mu/\lambda$ so that with $\lambda/\mu \geq
(1+\delta')/\varepsilon'$, for all $\gamma \in (0, \gamma_0]$. Then
\begin{align*}
  \beta(\gamma,P)
    & \geq \frac{\lambda \gamma}{\mu}
    = \sqrt\frac{\lambda^2}{\mu^2}\gamma
    = \sqrt\frac{\lambda}{\mu\gamma_0}\gamma
    \geq \sqrt\frac{1+\delta'}{\varepsilon'\gamma_0}\gamma.\end{align*}

3. Similar to the proof of Lemma~\ref{lem:M3}, we remark that
$
  \beta(\gamma,P)
            \geq 1 - \frac{1}{1 + \eta\gamma}
$,
thus the rest of the proof is similar to $k$-tournament selection.
\end{proof}

\section{Applications to Genetic Algorithms for Different Problems}\label{sec:app}

This section applies the results from the previous section to
derive bounds on the expected runtime of GAs for optimising
pseudo-Boolean functions and solving combinatorial Optimization
problems.

In what follows, by bitwise mutation operator we mean an operator
that given a bitstring~$x$, computes a bitstring~${y}$, where
independently of other bits, each bit~$y_i$ is set to~$1-x_i$ with
probability~$p_{\rm m}$ and with probability~$1-p_{\rm m}$ it is
set equal to~$x_i$. The tunable parameter~$p_{\rm m}$ is called a
{\em mutation rate}.

\subsection{Optimisation of pseudo-Boolean functions}

In this subsection, we consider the expected runtime of non-elitist
GAs in Algorithm~\ref{algo:GA} on the following functions,
\begin{align*}
  &\onemax(x)
    := \sum_{i=1}^{n} x_i = |x|_1 = \om(x), \\
  &\leadingones(x)
    := \sum_{i=1}^{n} \prod_{j=1}^{i} x_i = \lo(x), \\
  &\jump_r(x)
    := \begin{cases}
         n + 1     & \text{if } |x|_1=n\\
         r + |x|_1 & \text{if } |x|_1\leq n-r\\
         n - |x|_1 & \text{otherwise}
        \end{cases},\\
  &\rr_r(x)
    := \sum_{i=0}^{n/r-1} \prod_{j=1}^{r} x_{ir+j},\\
  &\linear(x)
    := \sum_{j=1}^n c_i x_i.
\end{align*}

Note that our results on these functions also hold for their
generalised classes,
\ie the meaning of $0$-bit and $1$-bit in each position can be
exchanged, and/or $x$ is rearranged according to a fixed permutation
before each evaluation. For $\linear$, \wolg we can assume
$c_1\geq c_2\geq\dots\geq c_n >0$\cite{bib:dl16}. We also consider the class
of $\ell$-$\unimodal$ functions, for which each function has exactly
$\ell$ distinctive fitness values $f_1<f_2<\dots<f_\ell$, and each
bitstring $x$ of the search space is either optimal or it has a
Hamming-neighbour $y$ with a better fitness, \ie $f(y)>f(x)$.

For a moderate use of crossover, \ie $p_\mathrm{c}=1-\Omega(1)$,
Corollary~\ref{cor:GA} is applicable and provides upper bounds
on the expected runtime for all these functions and classes.

\begin{theorem}\label{thm:GA-on-pseudo-boolean-func}
The expected runtime of the GA in Algorithm~\ref{algo:GA},
with $p_\mathrm{c} = 1 - \Omega(1)$
using any crossover operator,
a bitwise mutation with mutation rate $\chi/n$ for any fixed constant $\chi>0$
and one of the selection mechanisms: $k$-tournament selection, $(\mu,\lambda)$-selection,
linear or exponential ranking selection, with their parameters $k$, $\lambda/\mu$
and $\eta$ being set to no less than $(1+\delta)e^{\chi}/(1-p_c)$ where $\delta>0$
being any constant, is
\begin{itemize}
  \item $\bigO{n\lambda}$ on $\onemax$ if $\lambda\geq c\ln{n}$,
  \item $\bigO{n^2 + n\lambda\ln{\lambda}}$ on $\leadingones$ if $\lambda\geq c\ln{n}$,
  \item $\bigO{n\ell + \ell\lambda\ln{\lambda}}$ on $\ell$-$\unimodal$ if $\lambda\geq c\ln(\ell n)$,
  \item $\bigO{n^2 + n\lambda\ln{\lambda}}$ on $\linear$ if $\lambda\geq c\ln{n}$,
  \item $\bigO{\left(\frac{n}{\chi}\right)^r + n \lambda  + \lambda \ln{\lambda} }$ on $\jump_r$
  if $\lambda\geq cr\ln{n}$,
  \item $\bigO{\left(\frac{n}{\chi}\right)^r\ln\left(\frac{n}{r}\right) + \frac{n\lambda\ln{\lambda}}{r}}$ on $\rr_{r\geq 2}$
  if $\lambda\geq cr\ln{n}$,
\end{itemize}
for some sufficiently large constant $c$.
\end{theorem}
\begin{proof}
We apply Corollary~\ref{cor:GA} with the canonical partition $A_{j}:=\{j
\mid f(x) = j\}$ for all functions\footnote{The first level can be $A_0$
instead of $A_1$ for some functions but that does not matter as far as we
compute the sums correctly later on.}, except for $\linear$, the
fitness-based partition \cite{bib:dl16}: \[A_{j} := \left\{ x \mid
\sum_{i=1}^j c_i \leq f(x) < \sum_{i=1}^{j+1} c_i \right\}\] for $j \in
\{0\} \cup [n-1]$ and $A_n:=\{1^n\}$, is used.

The choices of $s_j$ and $s_*$ to satisfy (M1) are the following.
\begin{itemize}
  \item For $\onemax$, we set \[s_j := {n-j \choose 1}\left(\frac{\chi}{n}\right)\left(1
- \frac{\chi}{n}\right)^{n-1} = \Omega\left(\frac{n-j}{n}\right),\] \ie the
probability of flipping a $0$-bit while keeping all the other bits unchanged,
and $s_*:= \Omega\left(\frac{1}{n}\right)$.
  \item For $\leadingones$, $\ell$-$\unimodal$ and $\linear$, we set \[s_j :=
\frac{\chi}{n}\left(1 - \frac{\chi}{n}\right)^{n-1} =
\Omega\left(\frac{1}{n}\right) =: s_*,\] \ie the probability of flipping a
specific bit to create a Hamming neighbour solution with better fitness while
keeping all the other bits unchanged. In $\ell$-$\unimodal$, the bit to flip
must exist by the definition of the function. In $\leadingones$, the $0$-bit
at position $j+1$ should be flipped. For $\linear$, the partition satisfies
that among the first $j+1$ bits there must be at least a $0$-bit, thus it
suffices to flip the left most $0$-bit will produce a search point at a
higher level.
  \item For $\jump_r$, as similar to $\onemax$ for $j \in [n-1]$  we use
\[s_j :=
{n- j + 1\choose 1}\left(\frac{\chi}{n}\right)\left(1 - \frac{\chi}{n}\right)^{n-1}
= \Omega\left(\frac{n-j + 1}{n}\right),\]  but $s_{n} := \left(\frac{\chi}{n}\right)^r
\left(1-\frac{\chi}{n}\right)^{n-r} = \Omega\left(\left(\frac{\chi}{n}\right)^r\right)$,
\ie the probability of flipping the $r$ remained $0$-bits, so $s_* :=
\Omega\left(\left(\frac{1}{n}\right)^r\right)$.
  \item For $\rr_r$, we use \[s_j := {n/r - j \choose
1}\left(\frac{\chi}{n}\right)^r\left(1-\frac{\chi}{n}\right)^{n-r} =
\Omega\left(\left(\frac{\chi}{n}\right)^r \left(\frac{n}{r} - j\right)\right),\]
\ie the probability of flipping an entire unsolved block of length $r$ (in the
worst case) while keeping the other bits unchanged, and $s_* :=
\Omega\left(\left(\frac{1}{n}\right)^r\right)$.
\end{itemize}

It follows from Lemma~\ref{lem:prob-no-flip} that the probability of not
flipping any bit position by mutation is $(1-\chi/n)^n \geq \left(1 -
\frac{\delta/2} {1+\delta/2}\right)e^{-\chi} = \frac{e^{-\chi}}{1+\delta/2}$ for
$n$ sufficiently large, thus choosing $p_0 := \frac{e^{-\chi}}{1+\delta/2}$
satisfies (M2).

We now look at (M3). In $k$-tournament selection, we have \[k \geq
\frac{(1+\delta)e^{\chi}}{1-p_\mathrm{c}} =
\left(1+\frac{\delta/2}{1+\delta/2} \right)
\frac{1}{(1-p_\mathrm{c})p_0}.\] Hence, it follows from
Lemma~\ref{lem:M3} that (M3) is satisfied with constant
$\delta':=\frac{\delta/2}{1+\delta/2}$. The same conclusion can be
drawn for the other three selection mechanisms.

In (M4), since $\gamma_0$ and $\delta'$ are constants, there should exist
a constant $c>0$ for each function such that the condition is satisfied given
the minimum requirement on population size related to $c$.

Since all conditions are satisfied, Corollary~\ref{cor:GA} gives the desired
result for each function. For $\onemax$ and $\jump_r$, optimisation time can
be saved at early levels, \ie $s_j$ is not small at the beginning, thus the
evaluation of the sum $\sum_{j=1}^{m-1} \ln\left(\frac{6\delta\lambda}{4 +
\gamma_0 s_j \delta \lambda}\right)$ has to be precise:
\begin{itemize}
  \item For $\onemax$, simplifying each term by $\ln\left( \frac{6} {\gamma_0 s_j}
  \right)$ gives
      \[\bigO{ \ln\left(\frac{6^n n^n}{ \gamma_0^n \prod_{j=0}^{n-1} (n-j)}\right)}
      = \bigO{ \ln\left(\frac{6^n n^n}{n! \gamma_0^n}\right) },\]
    and by Stirling's approximation $n! = \Theta(n^{n+\frac{1}{2}} / e^{n})$, this is
  no more than $\bigO{n}$. The expected runtime is then $\bigO{\lambda n + n\ln{n}}$.
  Since we already require $\lambda = \Omega(\ln{n})$, this can be written shortly
  as $\bigO{n\lambda}$.
    \item For $\jump_r$, we use the simplification $\ln\left( \frac{6}{\gamma_0 s_j}
  \right)$ for the first $m-2$ terms of the sum, and $\ln(3\lambda/2)$ for the last
  term, so this gives
  \begin{align*}
       \bigO{ \ln\left(\frac{6^n n^{n-1}}{\gamma_0^n \prod_{j=1}^{n-1}(n-j+1)} \right) + \ln{\lambda} } 
       &= \bigO{ \ln\left(\frac{6^n n^n }{n! \gamma_0^n } \right) + \ln{\lambda} }\\
       = \bigO{n + \ln{\lambda}}.    
  \end{align*}
    The expected runtime is then
  $\bigO{\left(\frac{n}{\chi}\right)^r + n \lambda  + \lambda \ln{\lambda}}$.
\end{itemize}

For the other functions, $s_j$ is already small at early levels, thus
there is no benefit of considering the gradual sum of $\ln$. Hence, the
simplification $\bigO{m\ln{\lambda}}$ for the sum $\sum_{j=1}^{m-1} \ln\left(
\frac{6\delta\lambda}{4 + \gamma_0 s_j \delta \lambda}\right)$ gives
the corresponding results.
\end{proof}

In the case of regular use of crossover, \ie $p_\mathrm{c} = 1$, the
relationship between the crossover operator and the structure of the search space
becomes non-negligible. In the following,
we consider a general {\em mask-based} crossover as follows.
Given two parent
genotypes~$u,v,$
the operator consists in first choosing (deterministically or
randomly) a binary string~$\tilde{m} = (m_1,\dots,m_n)$
to produce two offspring vectors~$x',x''$ as
\[
\begin{array}{ll}

x'_i= \left\{
\begin{array}{ll}
u_i, & \mbox{ if } \ m_i=1\\
v_i, &  \mbox{ otherwise,}
\end{array} \right.

&

x''_i= \left\{
\begin{array}{ll}
v_i, & \mbox{ if } \ m_i=1\\
u_i, &  \mbox{ otherwise.}
\end{array} \right.
\end{array} \
\]
Then one element of $\{x',x''\}$ chosen uniformly at random is
returned. For example, the uniform crossover is a mask-based
crossover for which $\tilde{m} \sim \unif(\{0,1\}^n)$,
and a $k$-point crossover is a mask-based crossover for which
\[\tilde{m} \sim \unif\left(\left\{0^{a_1}1^{a_2} 0^{a_3}\dots { }
\mid a_i \in \mathbb{N}, \sum_{i=1}^{k+1}a_i=n\right\}\right).\]

The following lemma shows that
all mask-based crossover operators satisfy~(C3) with
$\varepsilon={1}/{2}$ for $\om$ and $\lo$ functions.

\begin{lemma} \label{lem:constant-epsilon-OM-LO}
If $x \sim \pxor(u,v)$, where~$\pxor$ is a
mask-based crossover, then:
\begin{enumerate}
\item
If $\lo(u)=\lo(v)=j$, then
 $\Pr\left(\lo(x) \ge j \right) =1$,\\
otherwise
$\Pr\left(\lo(x) > \min\{\lo(u),\lo(v)\}\right) \geq 1/2$.
\item $  \Pr\left(\om(x) \geq
  \lceil(\om(u) +\om(v))/2\rceil\right) \geq 1/2.$
\end{enumerate}
\end{lemma}

\begin{proof}

1) When $\lo(u)=\lo(v)=j$, 
in
mask-based crossover operators, the two
bitstrings $x'$, $x''$
have~$j$ leading ones. So does the
returned bitstring,
\ie with probability~$1$.

If $\lo(u) \neq \lo (v)$,
we can assume \wolg that $\lo(v) = j$ and $\lo(u)
> \lo(v)$. Then $v$ has a~$0$ while $u$ has a~$1$ at
position $j+1$. 
So,
one of the bitstrings $x'$, $x''$ in the mask-based crossover will inherit the~$1$
at that position and the other will inherit the~$0$. This implies
that one of
them
has fitness at least~$j+1$
and with probability~$1/2$ it is
returned as output.

2) 
Each bit of $u$ and $v$ is copied either to $x'$ or to $x''$,
therefore $|x'|_1+|x''|_1 = |u|_1+|v|_1$,
which means that $\max\{|x'|_1,|x''|_1\} \ge \lceil(|u|_1 +
|v|_1)/2\rceil$.
The output is chosen with
probabilities~1/2 to be copied either from~$x'$ or $x''$,
and the result follows.
\end{proof}

\begin{theorem}\label{thm:GA1-on-OM-LO}
  Assume that the GA in Algorithm~\ref{algo:GA}
  with $p_\mathrm{c}=1$
  uses any mask-based crossover operator,
  a bitwise mutation with mutation rate $\chi/n$ for any fixed constant $\chi>0$,
  and one of the following selection mechanisms:
  \begin{itemize}
   \item $k$-tournament selection with $k\geq 8(1+\delta)e^{\chi}$,
   \item $(\mu,\lambda)$-selection with $\lambda/\mu \geq 2(1+\delta)e^{\chi}$, or
   \item exponential ranking selection with $\eta \geq 8(1+\delta)e^{\chi}$,
  \end{itemize}
  for any constant $\delta>0$. Then there exists a constant~$c>0$,
  such that the expected runtime of the GA is
  \begin{itemize}
    \item $\bigO{n\lambda}$ on $\onemax$ if $\lambda\geq c \ln n$,
    \item $\bigO{n^2 + n\lambda \ln \lambda}$ on $\leadingones$ if $\lambda\geq c \ln n$.
  \end{itemize}
\end{theorem}

\begin{proof}
We apply Corollary~\ref{cor:GA1} this time, but again using the canonical
partition of the search space for both functions.
We also assume that $n$ is large enough so that by Lemma~\ref{lem:prob-no-flip}
the probability of not flipping any bit by mutation is $(1-\chi/n)^n \geq
\left(1 - \frac{\delta/2} {1+\delta/2}\right)e^{-\chi} = \frac{e^{-\chi}}
{1+\delta/2} =: p_0$, and so (C2) is satisfied with this choice of $p_0$.
In addition, we use the same upgrades probabilities $s_j$ and their smallest
value $s_*$ for each of the two functions as in the proof of
Theorem~\ref{thm:GA-on-pseudo-boolean-func} to satisfy (C1).

It follows from Lemma~\ref{lem:constant-epsilon-OM-LO} that (C3) is
satisfied for constant $\varepsilon_1:=1/2$.
We now look at condition (C4). For $k$-tournament, we get
$ k \geq 8(1+\delta)e^{\chi}
        = 4\left(1 + \frac{\delta/2}{1+\delta/2}\right) / (p_0\varepsilon_1)
    $.
So condition (C4) is satisfied
with constant $\delta':= \frac{\delta/2}{1+\delta/2}$ for $k$-tournament
by Lemma~\ref{lem:C4}. The same reasoning can be applied so that (C4) is also
satisfied for the other selection mechanisms.

Since $\delta'$ and $\gamma_0$ are constants, thus condition~(C5)
is satisfied given $\lambda \geq c\ln{n}$ and for some
constant~$c$.
Since all conditions are satisfied, the result follows
from Corollary~\ref{cor:GA1}.
\end{proof}

Note that the upper bounds in Theorem~\ref{thm:GA1-on-OM-LO} match
the upper bounds of Theorem~\ref{thm:GA-on-pseudo-boolean-func}.
The latter is a generalisation of the results in \cite{bib:dl16}
which were limited to EAs without crossover.

In the next sections, we further demonstrate the generality of
Theorem~\ref{thm:general-fitness-levels} through Corollary~\ref{cor:GA}
by deriving bounds on the expected runtime of GAs with $p_\mathrm{c}
= 1 - \Omega(1)$ to optimise or to approximate the optimal solutions of
combinatorial optimisation problems. We start with a simple problem
of sorting $n$ elements from a totally ordered set.

\subsection{Optimisation on permutation space}
Given $n$ distinct elements from a totally ordered set, we
consider the problem of ordering them so that some \emph{measure
of sortedness} is maximised. Several measures were considered
by~\cite{Scharnow2005} in the context of analysing the~\oneplusoneEA.
One of those is $\inv(\pi)$ which is defined to be the number of
pairs $(i, j)$ such that $1 \leq i< j \leq n$, $\pi(i) < \pi(j)$
(i.e. pairs in correct order).
We show that with the method introduced in this paper,
\ie Corollary~\ref{cor:GA} analysing GAs on Sorting problem with
$\inv$ measure, denoted by $\sortinginv$, is not much harder than
analysing the~\oneplusoneEA.

For the mutation we use the $\exchg(\pi)$
operator~\cite{Scharnow2005}, which consecutively applies $N$
pairwise exchanges between uniformly selected pairs of indices,
where $N$ is a random number drawn from a Poisson distribution
with parameter $1$.

\begin{theorem} \label{thm:GA-on-sorting}
If the GA in Algorithm~\ref{algo:GA}
with $p_\mathrm{c} = 1 - \Omega(1)$
uses any crossover operator,
the $\exchg$ mutation operator,
one of the selection mechanisms:
$k$-tournament selection, $(\mu,\lambda)$-selection, and linear or exponential ranking selection, with
their parameters $k$, $\lambda/\mu$ and $\eta$ being
set to no less than $(1+\delta)e/(1-p_c)$,
then there exists a constant
$c>0$ such that if the population size is $\lambda\geq c\ln n$,
the expected time to obtain the optimum of $\sortinginv$ is
$\bigO{n^2\lambda}$.
\end{theorem}
\begin{proof}
  Define $m:= \binom{n}{2}.$ We apply
  Corollary~\ref{cor:GA} with the canonical partition,
  $A_{j} : = \{\pi \mid \inv(\pi) = j\}$ for $j=\{0\} \cup [m]$.
  The probability that mutation exchanges $0$ pairs is $1/e$.
  Hence, condition (M2) is satisfied
  for $p_0:=1/e$.

  To show that (M1) is satisfied, we first define $s_j :=
  \frac{m-j}{em}$ for each $j \in \{0\}\cup [m-1]$. Since
  $x\in A_j$, then the probability that the exchange operator exchanges
  exactly one pair is $1/e$, and the
  probability that this pair is incorrectly ordered in $x$, is
  $(m-j)/m$.
  Thus, (M1) is satisfied with the defined $s_j$.

  In (M3), for $k$-tournament we have that $k \geq \frac{(1+\delta)e}
  {1-p_\mathrm{c}} = \left(1 + \frac{\delta/2}{1+\delta/2}\right)\frac{1}
  {(1-p_\mathrm{c})p_0}$, thus the condition is satisfied for constant
  $\delta':=\frac{\delta/2}{1+\delta/2}$ and some constant $\gamma_0\in(0,1)$
  by Lemma~\ref{lem:M3}. The same conclusion can be drawn be the other
  selection mechanisms.
    Finally, since $\gamma_0, \delta'$ are constants, there exists a constant
  $c>0$ such that (M4) is satisfied for any $\lambda \geq c\ln(n)$.

  It therefore follows by Corollary~\ref{cor:GA} that the expected
  runtime of the GA on $\sortinginv$ is~$\bigO{n^2\lambda}$, \ie this
  is similar to $\onemax$ except that we have $m=O(n^2)$ levels.
\end{proof}

\subsection{Search for Local Optima}\label{sec:cop}

A great interest in the area of combinatorial optimisation is
to find approximate solutions to NP-hard problems, because exact
solutions for such problems are unlikely be computable in polynomial
time under the so-called P$\ne$NP hypothesis. In the case of
maximisation problems, a feasible solution is called a {\em
$\rho$-approximate solution} if its objective function value is at
least $\rho$ times the optimum for some $\rho\in (0,1]$. Local search
is one method among others to approximate solutions for combinatorial
optimisation problems through finding local optima (a formal definition
is given below).
For a number of well-known problems it was shown~\cite{AP95}
that any local optimum is guaranteed to be a $\rho$-approximate solution
with a constant~$\rho$.

Suppose that a {\em neighbourhood} ${\mathcal N}(x)\subseteq
{\mathcal X}$ is defined for every~$x\in {\mathcal X}$. The
mapping ${\mathcal N}: {\mathcal X} \to 2^{\mathcal X}$ is called
the {\em neighbourhood mapping} and all elements of ${\mathcal
N}(x)$ are called {\em neighbours} of~$x$.
For example, a frequently used neighbourhood mapping in the case of
binary search space~${\mathcal X}=\{0,1\}^n$ is defined by the Hamming
distance $H(\cdot,\cdot)$ and a radius~$r$ as
\[
{\mathcal N}(x)=\{y \mid H(x,y) \leq r\}.
\]
If $f(y)\leq f(x)$ holds for all neighbours~$y$ of~$x \in {\mathcal
X}$, then~$x$ is called a \emph{local optimum} \wrt~${\mathcal N}$. The
set of all local optima is denoted by~$\mathcal{LO}$ (note that
global optima also belong to~$\mathcal{LO}$).

A local search method starts from some initial solution~$y_0$.
Each iteration of the algorithm consists in moving from the
current solution to a new solution in its neighbourhood, so that
the value of the fitness function is increased. The way to choose
an improving neighbour, if there are several of them, will not
matter in this paper. The algorithm continues until a local
optimum is reached.
Let $m$ be the number of different fitness values attained by
solutions from~$\mathcal{X} \backslash \mathcal{LO}$ plus~1. Then
starting from any point, the local search method finds a local
optimum within at most~$m-1$ steps.

Alternatively, one can use GAs to solve the optimisation problem,
and possibly find local optima. The following result provides sets
of sufficient conditions for a performance guaranteed GA to find
local optima.

\begin{corollary}\label{cor:GA-as-LS1}
Given some positive constants $p_0, \varepsilon_0$ and
$\delta$, define the following conditions:
\begin{description}
\item[(X1)] $\displaystyle \pmut(y \mid x)\ge s$ for any $\ x \in {\mathcal X}, \ y \in
{\mathcal N} (x)$.
\item[(X2)] $\displaystyle \pmut(x \mid x)\geq p_0$ for all $x \in {\mathcal X}$.
\item[(X3)] $\displaystyle \pxor\big(f({x}') \ge \max\{f({u}),f({v})\} \mid u, v\big)\geq \varepsilon_0$
for any ${u},{v} \in {\mathcal X}$.
\item[(X4.1)] the non-elitist GA in Algorithm~\ref{algo:GA} is set with $p_\mathrm{c}=1$,
            and it uses one of the following selection mechanisms:
           \begin{itemize}
           \item $k$-tournament selection with ${k\ge\frac{4(1+\delta)}{\varepsilon_0 p_0}}$,
           \item $(\mu,\lambda)$-selection with $\frac{\lambda}{\mu} \ge \frac{(1+\delta)}{\varepsilon_0 p_0}$,
           \item exponential ranking selection with $\eta\ge\frac{4(1+\delta)}{\varepsilon_0 p_0}$.
           \end{itemize}
\item[(X4.2)] the non-elitist GA is set with $p_\mathrm{c}=1 - \Omega(1)$,
              and it uses one of the following selection mechanisms:
                                 $k$-tournament selection,                       $(\mu,\lambda)$-selection,                       linear or exponential ranking selection,                       with their parameters $k$, $\lambda/\mu$ and $\eta$ being set to no less than $\frac{(1+\delta)}{(1-p_\mathrm{c})p_0}$.
           \end{description}
If (X1-3) and (X4.1) hold, or exclusively (X1-2) and (X4.2)
hold, then there exists a constant~$c$, such that for~$\lambda \ge
c\ln\left(\frac{m}{s}\right)$, a local optimum is reached for the
first time after~$\bigO{m\lambda \ln \lambda + \frac{m}{s}}$
fitness evaluations in the expectation.
\end{corollary}

Condition (X4.1) or (X4.2) characterises the setting of selection
mechanisms, while (X1-3) bear the properties of the variation
operators over the neighbourhood structure $\mathcal{N}$.
Particularly, (X1) assumes
a lower bound $s$ on the probability that the mutation
operator transforms an input solution into a specific neighbour. To illustrate this condition,
we note that in most of the local search algorithms the neighbourhood
$\mathcal{N}({x})$ may be enumerated in polynomial time of the
problem input size. For such neighbourhood mappings, a mutation
operator that generates the uniform distribution
over~$\mathcal{N}({x})$ will select any given point
in~$\mathcal{N}(x)$ with probability at least~$s,$ so that $1/s$
is polynomially bounded in the problem input size.

If crossover is frequently used, \ie $p_\mathrm{c} = 1$, we also
need to satisfy condition (X3) on the the crossover operator.
It requires that the fitness of solution~$x$ on the output of
crossover is not less than the fitness of parents with probability
at least~$\varepsilon_0$.
Note that such a requirement is satisfied
with $\varepsilon_0 = 1$ for the optimized crossover operators,
where the offspring is computed as a solution to the {\em optimal
recombination problem} (see e.g~\cite{ErKov14I}). This
supplementary problem is known to be polynomially solvable for
Maximum Clique~\cite{BN98}, Set Packing, Set Partition and some
other NP-hard problems~\cite{ErKov14I}.

When a set of conditions, \ie depending on whenever $p_\mathrm{c} =
1$ or $p_\mathrm{c} = 1 - \Omega(1)$, is satisfied and given a
sufficiently large population \wrt to $m$ and $s$, an upper bound
on the expected number of fitness evaluations that GA performs to
find a local optimum is guaranteed. The proof directly follows from
Corollaries~\ref{cor:GA1} and \ref{cor:GA} of
Theorem~\ref{thm:general-fitness-levels}.

\begin{proof}[Proof of Corolary~\ref{cor:GA-as-LS1}]
We use the following partition
\begin{align*}
A_j &:= \{x\in \mathcal{X} | f(x) = f_j\} \backslash \mathcal{LO},
\ j\in [m-1], \text{ and }\\
A_{m} &:= \mathcal{LO}.
\end{align*}
We note that $(A_1,\dots,A_{m-1})$ is a fitness-based partition of
$\mathcal{X} \setminus \mathcal{LO}$. Thus, applications of
Corollary~\ref{cor:GA} and Lemma~\ref{lem:M3} for the set of
conditions (X1-2) and (X4.2), or alternatively,
Corollary~\ref{cor:GA1} and Lemma~\ref{lem:C4} for the set of
conditions (X1-3) and (X4.1), yield the required result.
\end{proof}

A similar result for GAs with very high selection pressure was
obtained in~\cite{bib:erarxiv16}. In particular, the result
from~\cite{bib:erarxiv16} implies that a GA with tournament
selection or $(\mu,\lambda)$-selection, given certain settings of
parameters, reaches a local optimum
after~$\bigO{{m\ln (m)}/{s}}$ fitness evaluations
in expectation. The upper bound from Corollary~\ref{cor:GA-as-LS1}
has an advantage to the bound from~\cite{bib:erarxiv16} if $1/s$
is at least linear in~$m$.

The effect of the corollary can be seen in the following example
setting. Let us consider the binary search
space~$\{0,1\}^n$ with Hamming neighbourhood mapping of a constant
radius~$r$, a fitness function~$f$ such that $m \in \poly(n)$,
and assume that GA uses bitwise mutation
operator and $p_c=1-\Omega(1)$. The bitwise mutation operator
outputs a string~${y}$, given a string~${x}$, with probability
$p_{\rm m}^{H({x},{y})}(1-p_{\rm m})^{n-H({x},{y})}$. Note that
probability~$p_{\rm m}^j(1-p_{\rm m})^{n-j}$, as a function
of~$p_{\rm m}$ attains its maximum at $p_{\rm m}=j/n$. It is easy
to show (see e.g.~\cite{bib:erarxiv16}) that for any ${x} \in
{\mathcal X}$ and ${y} \in \mathcal{N}({x})$, the bitwise mutation
operator with~$p_{\rm m}=r/n$ satisfies the condition $\pmut(y
\mid x)=\bigO{1/n^r}$. Besides that, for a sufficiently large~$n$
and any $x\in {\mathcal X}$ holds $\pmut(x \mid x)\geq
e^{-r}/2=\Omega(1)$. Therefore, Corollary~\ref{cor:GA-as-LS1}
implies that a GA with the above mentioned
operators
given appropriate settings of parameters
$\lambda, p_{\rm m}$ and $p_c$, first visits a local optimum
\wrt a Hamming neighbourhood of constant radius after a
polynomially bounded number of fitness evaluations in expectation.

To give concrete examples, we consider the following unconstrained
(and unweighted) problems:
\begin{itemize}
  \item \maxsat: given a CNF formula in~$n$ logical variables which
  is represented by $m'$ clauses~${\bf c}_1,\dots,{\bf c}_{m'}$ and
  each clause is a disjunction of logical variables or their negations,
  it is required to find an assignment of the variables so that the
  number of satisfied clauses is maximised.
  \item \maxcut: given an undirected graph $G=(V,E)$, it is required
  to find a partition of $V$ into two sets $(S,V\setminus S)$, so that
  the number of crossing edges, \ie $\delta(S):=|\{(u,v)
  \mid (u,v) \in E, u \in S, v \notin S\}|$, is maximised.
\end{itemize}

Both problems are NP-hard, and their solutions can be naturally
represented by bitstrings. Particularly, any local optimum \wrt
the neighbourhood defined by Hamming distance~$1$ has at least
half the optimal fitness~\cite{AP95}. Better approximation ratios
can be obtained with more sophisticated algorithms, \eg $0.79$-approximation
for \maxsat~\cite{ABZ05} based on a time-consuming semi-definite
programming relaxation. The local search algorithm with the above
neighbourhood however has an advantage of low time complexity, \eg
it only makes $\bigO{nm'}$~tentative solutions for \maxsat.
Corollary~\ref{cor:GA-as-LS1} translates such a result into
relatively low runtime bound for GAs. 

\begin{theorem}\label{thm:GA-on-approx-MAXSAT}
Suppose the GA in Algorithm~\ref{algo:GA} is applied to \maxsat or to
\maxcut using a bitwise mutation with~$p_{\rm m}=\chi/n$, where~$\chi>0$
is a constant, a crossover with $p_c=1-\Omega(1)$ and one of the selection
mechanisms: $k$-tournament selection, $(\mu,\lambda)$-selection, linear
or exponential ranking selection, with their parameters $k$, $\lambda/\mu$
and $\eta$ being set to no less than $\frac{(1+\delta)e^{\chi}}{1-p_c}$,
where $\delta>0$ is any constant.
Then there exists a constant~$c$, such that for~$\lambda \ge
c\ln(nm')$, a $1/2$-approximate solution is reached for the first
time after~$\bigO{m' \lambda \ln \lambda + nm'}$ fitness
evaluations in expectation for $\maxsat$, and after $\bigO{|E|
\lambda \ln \lambda + |V|\;|E|}$ for \maxcut.
\end{theorem}

The proof is analogous to the analysis of $\ell$-$\unimodal$
function in Theorem~\ref{thm:GA-on-pseudo-boolean-func}, combined
with Corollary~\ref{cor:GA-as-LS1} where $m$ is no more than
$m'+1$ for \maxsat and no more than $|E|+1$ for \maxcut.

\section{Estimation of Distribution Algorithms}\label{sec:umda}

As mentioned in the introduction, there are few rigorous runtime
results for UMDA and other estimation of distribution algorithms
(EDAs). The analytic techniques used in previous analyses of EDAs were
often complex, \eg relying on the machinery of Markov chain
theory. Surprisingly, even apparently simple problems, such as the
expected runtime of UMDA on \onemax, were until recently open.

Algorithm~\ref{algo:Algorithm1} matches closely the typical behaviour
of estimation of distribution algorithms: given a current distribution
over the search space, sample a finite number of search points, and
update the probability distribution. We demonstrate the ease at which
the expected runtime of UMDA with margins and truncation selection on
the \onemax function can be obtained using the level-based theorem
without making any simplifying assumptions about the optimisation
process.

\subsection{Algorithm}

If $P \in \mathcal{X}^\lambda$ is a population of $\lambda$
solutions, let $P(k,i)$ denote the value in the $i$-th bit
position of the $k$-th solution in $P$. The Univariate Marginal
Distribution Algorithm (UMDA) with $(\mu,\lambda)$-truncation
selection is defined in Algorithm \ref{algo:UMDA}.

The algorithm has three parameters, the parent population size
$\mu$, the offspring population size $\lambda$, and a parameter
$m'<\mu$ controlling the size of the margins. It is necessary to
set $m'>0$ to prevent a premature convergence, \eg
without this margin $p_t(i)$ can go to a non-optimal fixation,
this prevents further exploration and causes an infinite runtime.
Based on insights about optimal mutation rates in the (1+1) EA, we
will use the parameter setting $m'=\mu/n$ in the rest of this
section.

It is immediately clear that the UMDA in Algorithm~\ref{algo:UMDA}
is a special case of Algorithm~\ref{algo:Algorithm1} scheme. The
probability distribution $D(P_t)$ of~$y$ is computed in steps 6-7,
and is defined for any search point $x\in\{0,1\}^n$ by
\begin{align*}
  \prob{y=x} = \prod_{j=1}^n p_t(j)^{x_j}\left(1-p_t(j)\right)^{1-x_j}.
\end{align*}

\begin{algorithm}[H]
  \caption{UMDA}
  \begin{algorithmic}[1]
    \STATE Initialise the vector $p_0 := (1/2,\ldots,1/2)$.
    \FOR{$t= 1, 2, 3, \ldots$}     \FOR{$x=1$ to $\lambda$}
    \STATE \label{algo:UMDA:sampling}Sample the $x$-th individual $P_t(x,\cdot)$ according to
    \begin{align*}
      P_{t}(x,i) \sim \bernoulli(p_{t-1}(i))\text{ for all } i\in[n].
    \end{align*}
    \ENDFOR
    \STATE Sort the population $P_{t}$ according to $f$.
    \STATE Calculate a new vector $p_{t}$ from $P_{t}$ according to
      \begin{align*}
                p_{t}(i)
          & :=
          \begin{cases}
             \displaystyle \frac{m'}{\mu}      & \text{ if }     X_i < m'\\
             \displaystyle \frac{X_i}{\mu}    & \text{ if }  m' \leq X_i \leq \mu - m' \\
             \displaystyle 1 - \frac{m'}{\mu}  & \text{ if }  \mu - m' < X_i,
          \end{cases}
        \end{align*}\\
        for all $i\in[n]$, where $X_i := \sum_{k=1}^{\mu} P_{t}(k,i)$.
    \ENDFOR
  \end{algorithmic}
  \label{algo:UMDA}
\end{algorithm}

Note that the sampling from vector $(p_i)_{i\in[n]}$ in UMDA is
analogous to a population-wise crossover, \ie the bit sampling
at position $i$ with a non-marginal probability (line~\ref{algo:UMDA:sampling}
in Algorithm~\ref{algo:UMDA}) is equivalent to picking uniformly
at random a bit from the $\mu$ bits at position $i$ of the $\mu$
selected individuals of the previous generation.
In some other randomised search heuristics such as ant
colony optimisation (ACO) and compact genetic algorithms (cGA),
the sampling distribution~$D_t$ does not only depend on the
current population, but also on additional information, such as
pheromone values. The level-based theorem does not apply to such
algorithms.

It is well-known that the (1+1) EA solves \onemax problem in
expected time $\Theta(n\ln n)$, and this is optimal for the class
of unary, unbiased black-box algorithms. Surprisingly, no previous
runtime analysis of UMDA seems available for \onemax. We
demonstrate that the expected runtime can be obtained relatively
easy with our methods. To obtain lower bounds on the tail of the
level-distribution, we make use of the Feige inequality~\cite{bib:Feige2004}
(or see Lemma~\ref{lem:feige-ineq} in the appendix).

\begin{theorem}
  Given any positive constants     $\delta\in(0,1)$,
 and $\gamma_0 \leq \frac{1}{(1+\delta)13e}$,
 the
   UMDA with
    offspring population size $\lambda$ with $b\ln(n)\leq \lambda\leq n/\gamma_0$ for some
    constant $b>0$,
    parent population size $\mu=\gamma_0\lambda$
            and margins $m'=\mu/n$,
  has expected optimisation time
  $\bigO{n\lambda\ln\lambda}$ on \onemax.
\end{theorem}
\begin{proof}
Step 1:
We use the canonical partition into $m=n+1$ levels,
where level $j\in[m]$ is defined by
\begin{align*}
  A_j := \{x\in\{0,1\}^n\mid  \onemax(x)=j-1 \}.
\end{align*}
We use the parameter $\gamma_0:=\mu/\lambda$ and let $Y$ be
the number of one-bits in the sampled solution.

The choice $m' = \mu/n$ and $\mu\leq n$ implies that the margins for
$p_t(i)$ are simplified to $1/n$ and $1-1/n$, and that these margins
are only used when the bit values at position $i$ of the $\mu$
selected individuals are identical.
We categorise the probabilities
$p_t(i)$ into three groups:
those at the upper margin $1-1/n$,
those
at the lower margin $1/n$, and
intermediary values
in the closed interval $[1/\mu,1-1/\mu]$.  Due to linearity of the fitness
function, the components of $p_t$ can be rearranged without changing
the distribution of $Y$. We assume w.l.o.g. a rearrangement so that there
exists integers $k,\ell\geq 0$ satisfying
\begin{align*}
  1\leq X_i&<\mu\quad\text{and}\quad p_t(i)=X_i/\mu &\text{if } &1\leq i\leq k,\\
        X_i&=\mu\quad\text{and}\quad p_t(i)=1-1/n &\text{if } &k<i\leq k+\ell, \text{and}\\
        X_i&=0\quad\text{and}\quad p_t(i)=1/n &\text{if } &k+\ell<i\leq n.
\end{align*}
By these assumptions, it follows that
\begin{align}
  \sum_{i=k+1}^{k+\ell}X_i = \mu\ell\quad\text{ and }\quad\sum_{i=k+\ell+1}^{n}X_i = 0.\label{eq:1}
\end{align}

In the following, we define $Y_{i,k}$ to be the number of
sampled one-bits due to
$(p_t(i),\ldots,p_t(k))$ in the rearranged $p_t$.

For any population $P_t$ and any $\gamma\in[0,\gamma_0]$, let $j\in[n]$ be any integer
such that $|P_t\cap A_{\ge j} |\geq \gamma_0\lambda=\mu$
and $|P_t\cap A_{\ge j+1} |\geq\gamma\lambda$.
This implies that among the $\mu$
fittest individuals in the current population, there are at least
$\gamma\lambda$ individuals with at least~$j$ one-bits, and
the remaining among the $\mu$ fittest individuals have at least
$j-1$ one-bits. Hence, the total number of one-bits among the
fittest $\mu$ individuals must satisfy
\begin{align}
  \sum_{i=1}^n X_i & \geq \gamma\lambda j+(\mu-\gamma\lambda)(j-1) =
  \gamma\lambda+\mu (j-1).\label{eq:2}
\end{align}
Combining Eqs. (\ref{eq:1}) and (\ref{eq:2}), when $k\geq 1$, we get
\begin{align}
  \expect{Y_{1,k}} = \sum_{i=1}^k p_t(i)=\frac{1}{\mu} \sum_{i=1}^k X_i \geq
  \frac{\gamma\lambda}{\mu}+ j-1-\ell.  \label{eq:3}
\end{align}

Step 2:
We first verify condition~(G2), \ie checking if $\prob{Y \geq j}
\geq (1+\delta)\gamma$ for any level $j$ defined like above with
$\gamma>0$.
We distinguish between two cases, either $k=0$ or $k\geq 1$.

\underline{Case 1}: If $k\geq 1$, then
Eq.~(\ref{eq:3}) and Lemma~\ref{lem:feige-ineq} give
\begin{align*}
  \prob{Y_{1,k}\geq j-\ell}
   & =    \prob{Y_{1,k}> j-1-\ell+\frac{\gamma\lambda}{\mu}-\frac{\gamma\lambda}{12\mu}}\\
   & \geq \prob{Y_{1,k}> \expect{Y_{1,k}}-\frac{\gamma\lambda}{12\mu}}\\
   & \geq \min\left\{\frac{1}{13},\frac{\frac{\gamma\lambda}{12\mu}}{\frac{\gamma\lambda}{12\mu}+1}\right\}\\
   & =    \min\left\{\frac{1}{13},\frac{\gamma\lambda}{\gamma\lambda+12\mu}\right\}
   \geq \frac{\gamma\lambda}{13\mu}.
\end{align*}

The probability of sampling an individual with at least~$j$
one-bits in the next generation is therefore lower-bounded by
\begin{align*}
  \prob{Y\geq j}
  & \geq \prob{Y_{1,k}\geq j-\ell}\prob{Y_{k+1,k+\ell}=\ell}\\
  & \geq \frac{\gamma\lambda}{13\mu}\left(1-\frac{1}{n}\right)^{\ell}
  \geq \frac{\gamma\lambda}{13\mu}\left(1-\frac{1}{n}\right)^{n-1}\\
  & \geq \frac{\gamma\lambda}{13e\mu}
   \geq (1+\delta)\gamma.
\end{align*}

\underline{Case 2:} If $k=0$, then all the $\mu$ best individuals in
the population must be  identical. By assumption, there are
$\gamma\lambda\geq 1$ individuals with at least $j$ 1-bits, hence all
the $\mu$ best individuals must have at least $j$ 1-bits. In this case,
there are $\ell\geq j$ probabilities at the upper margin, and
we get
\begin{align*}
  \prob{Y\geq j} &\geq
  \prob{Y_{k+1}^{k+j}\geq j}\\
    &= \left(1-\frac{1}{n}\right)^j
    \geq \frac{1}{e}
    \geq 13\gamma_0(1+\delta)
    >    (1+\delta)\gamma,
\end{align*}
and condition (G2) is therefore satisfied also in this case.

Step 3: We now consider condition~(G1) for
any $j$ defined with $\gamma = 0$.
Again we check the two cases $k=0$, and $k\geq 1$.

\underline{Case 1:} If
$k=0$, then with our assumption, the $\ell\geq j-1$ first probabilities are
at the upper margin $1-1/n$, and the last $n-\ell\leq n-j+1$ probabilities
are at the lower margin $1/n$. In order to obtain a search point
with at least $j$ one-bits, it is sufficient to sample exactly
$\ell\geq j-1$
one-bits in the first $\ell$ positions and exactly one 1-bit in
the last $n-\ell\leq n-j+1$ positions. Hence,
\begin{align*}
  \prob{Y\geq j}
  & \geq \prob{Y_{1,\ell}\geq \ell}\prob{Y_{\ell+1,n}\geq 1}\\
  & \geq \left(1-\frac{1}{n}\right)^\ell \left(\frac{n-\ell}{n}\right)\\
  & \geq    \left(1-\frac{1}{n}\right)^{n-1}\left(\frac{n-j+1}{n}\right)
  \geq  \left(\frac{n-j+1}{en}\right).
\end{align*}

\underline{Case 2:} When $k\geq 1$, we note from Eq.~(\ref{eq:3}) that
\begin{align*}
  \expect{ Y_{2,k} } = \expect{Y_{1,k}}-p_t(1) \geq j-1-\ell-p_t(1)
\end{align*}
Again, by Lemma~\ref{lem:feige-ineq}
we get
\begin{multline*}
  \prob{  Y_{2,k}\geq j-1-\ell}\\
   =   \prob{  Y_{2,k} > j-1-\ell-p_t(1) - (1 - p_t(1))} \\
   \geq \prob{  Y_{2,k} > \expect{Y_{2,k}} - (1 - p_t(1))} \\
   \geq \min\left\{\frac{1}{13}, \frac{1-p_t(1)}{2-p_t(1)}\right\}
   >    \frac{1-p_t(1)}{13}.
\end{multline*}
The probability of sampling an individual with at least $j$
one-bits in this configuration is bounded from below as
\begin{multline*}
  \prob{Y\geq j} \\
   > \prob{Y_1=1}\prob{Y_{2,k}\geq j-1-\ell}\prob{Y_{k+1,k+\ell}=\ell}\\
   \geq \left(\frac{p_t(1)(1-p_t(1))}{13}\right)\left(1-\frac{1}{n}\right)^\ell\\
   \geq \left(\frac{(1/\mu)(1-1/\mu)}{13}\right)\left(1-\frac{1}{n}\right)^\ell
   \geq \frac{1}{14e\mu}.
\end{multline*}
The last inequality holds for $\mu\geq 14$, which in turn only
requires $n$ to be larger than some constant. Hence, combining the
cases $k=0$ and $k>0$,
for all $j\in [n]$ we get
\begin{align*}
  \prob{Y\geq j}&\geq \min\left\{\frac{1}{14e\mu},
    \frac{n-j+1}{en}\right\}\\
&\geq \frac{n-j+1}{14e\mu(n-j+1)+en} =: z_{j}.
\end{align*}
Clearly, there exists a $z_*$ with $1/z_*\in\poly(n)$ such that
$\prob{Y\geq j}\geq z_*$ for all $j\in[n]$ and condition~(G1) is
satisfied.

Step 4: We consider condition~(G3) regarding the population size.
The parameters $\delta$ and $\gamma_0=\mu/\lambda$ are constants
with respect to $n$, therefore the variables $a, \varepsilon$ and
$c$ in condition~(G3) are also constants, and $1/z_*\in\poly(n)$.
Hence, there must exist a constant $b>0$ such that condition~(G3)
is satisfied when $\lambda \geq b\log(n).$

Step 5: To conclude, the expected optimisation time is
\begin{align*}
  &\bigO{ n\lambda\ln(\lambda)+\sum_{j=1}^{n} \frac{1}{z_j} }\\
   &=  \bigO{ n\lambda\ln(\lambda)+14e\mu n + \sum_{i=0}^{n-1} \frac{en}{n-i} }
    = \bigO{ n\lambda\ln \lambda }.\qedhere
\end{align*}
\end{proof}

A similar analysis for \leadingones~\cite{DangLehre2015} yields an upper
 bound on the runtime of $\bigO{n\lambda\ln \lambda+n^2}$ with
offspring population size $\lambda \ge b \ln(n)$ for some
constant~$b>0$ without use of Feige's inequality. The previous
result~\cite{bib:Chen2010} on \leadingones requires a larger
population size and gives a longer runtime bound.

Table \ref{tabl:summary2} summarises the runtime bounds for the
example applications of the tools presented in this paper and
the above mentioned result for UMDA on \leadingones.

\ifarxiv
\begin{landscape}
\fi
\begin{table*}[t]
\renewcommand{\arraystretch}{1.7}
\setlength{\tabcolsep}{0.4em}
\caption{Summary of results for GA (Algo.~\ref{algo:GA} with $p_\mathrm{c}=1-\Omega(1)$), GA1 (Algo.~\ref{algo:GA} with $p_\mathrm{c}=1$) and UMDA (Algo.~\ref{algo:UMDA} with margin $m'=\mu/n$).\label{tabl:summary2}}\vspace{-1em}
\begin{center}
\begin{tabular}{@{}c@{\hskip 3em}c@{}}
\begin{tabular}[t]{@{}llll@{}}
\multicolumn{4}{@{}l@{}}{\sc Runtime result}\\
\toprule
  {\bf Problem}
& {\bf Algorithm}
& {\bf Min. $\boldsymbol{\lambda}$}
& {\bf Runtime}\\
\midrule
  \multirow{2}{*}{$\onemax$}
  & GA, GA1
  & $c\ln{n}$
  & $\bigO{n\lambda}$\\

  & UMDA
  & $c\ln{n}$
  & $\bigO{n\lambda\ln{\lambda}}$\\

  $\leadingones$
  & GA, GA1, UMDA
  & $c\ln{n}$
  & $\bigO{n^2 + n\lambda\ln{\lambda}}$\\

  $\ell$-$\unimodal$
  & GA
  & $c\ln(n \ell)$
  & $\bigO{n \ell + \ell\lambda\ln{\lambda}}$\\

  $\linear$
  & GA
  & $c\ln{n}$
  & $\bigO{n^2 + n\lambda\ln{\lambda}}$\\

  $\jump_r$
  & GA
  & $c r \ln{n}$
  & $\bigO{\left(\frac{n}{\chi}\right)^r + n\lambda + \lambda \ln{\lambda}}$\\

  $\rr_{r\geq 2}$
  & GA
  & $c r \ln{n}$
  & $\bigO{\left(\frac{n}{\chi}\right)^r \ln\left(\frac{n}{r}\right) + \frac{n\lambda\ln{\lambda}}{r}}$\\

  $\sortinginv$
  & GA
  & $c\ln{n}$
  & $\bigO{n^2\lambda}$\\

  $\frac{1}{2}$-approx. \maxsat
  & GA
  & $c \ln(n m')$
  & $\bigO{nm' + m'\lambda\ln{\lambda}}$\\

  $\frac{1}{2}$-approx. \maxcut
  & GA
  & $c \ln(|V|\;|E|)$
  & $\bigO{|V|\;|E| + |E|\lambda\ln{\lambda}}$\\
\bottomrule
\end{tabular} &
\begin{tabular}[t]{@{}llll@{}}
\multicolumn{4}{@{}l@{}}{\sc Configuration}\\
\toprule
  {\bf Alg.}
& {\bf Recomb.}
& {\bf Selection}
& {\bf Setting}\\

\midrule
{\bf GA}
  & any
  & $k$-tournament
  & $k \geq \frac{(1+\delta)e^{\chi}}{1-p_\mathrm{c}}$\\

  & any
  & $(\mu,\lambda)$-selection
  & $\frac{\lambda}{\mu} \geq \frac{(1+\delta)e^{\chi}}{1-p_\mathrm{c}}$\\

  & any
  & linear ranking
  & $\eta \geq \frac{(1+\delta)e^{\chi}}{1-p_\mathrm{c}}$\\

  & any
  & exp. ranking
  & $\eta \geq \frac{(1+\delta)e^{\chi}}{1-p_\mathrm{c}}$\\

\midrule
{\bf GA1}
  & mask-based
  & $k$-tournament
  & $k \geq 8(1+\delta)e^{\chi}$\\

  & mask-based
  & $(\mu,\lambda)$-selection
  & $\frac{\lambda}{\mu} \geq 2(1+\delta)e^{\chi}$\\

  & mask-based
  & exp. ranking
  & $\eta \geq 8(1+\delta)e^{\chi}$\\

\midrule
{\bf UMDA}
  & n/a
  & $(\mu,\lambda)$-selection
  & $\frac{\lambda}{\mu} \geq 13(1+\delta)e$\\

\bottomrule
\end{tabular}
\end{tabular}
\end{center}
\vspace{0.4em}

On $\{0,1\}^n$, GA and GA1 use bitwise mutation operator with rate
$\chi/n$ where $\chi$ is any constant. On permutation search
space, \ie Sorting, GA uses $\exchg$ mutation and its setting
assumes $\chi=1$.
 In the case of \maxsat, $n$ is the number of
logical variables and $m'$ is the number of clauses.
Parameter $\delta$ is any positive constant, and $c$ is some constant.

\end{table*}
\ifarxiv
\end{landscape}
\fi

\section{The level-based theorem is almost tight}\label{sec:tightness}

How accurate are the time bounds provided by the level-based theorem?
To answer this question, we first interpret the theorem as a
universally quantified statement over the operators $D$ satisfying the
conditions of the theorem. More formally, given a choice of
level-partitioning and set of parameters
$z_1,\ldots,z_{m-1},\delta,\gamma_0$, which we collectively denote by
$\Theta$, the theorem can be expressed on the form
\begin{align}
  \forall D\in\mathcal{D}_\Theta\; \expect{T_D}\leq t_\Theta,\label{eq:lb:1}
\end{align}
where $\mathcal{D}_\Theta$ is the set of operators $D$ in Algorithm
\ref{algo:Algorithm1} that satisfy the conditions of the level-based theorem
with parameterisation $\Theta$, $\expect{T_{D}}$ is
the expected running time of Algorithm~\ref{algo:Algorithm1} with
a given operator $D$, and $t_\Theta$ is the upper time bound provided by the
level-based theorem which depends on the parameterisation $\Theta$.

In order to obtain an accurate bound for a specific operator $D$, for
example the ($\mu$,$\lambda$) EA applied to the \onemax function, it
is necessary to choose a parameterisation $\Theta$ that reflects this
process as tightly as possible. If the bounds on the
``upgrade'' probabilities $z_j$ for the ($\mu$,$\lambda$) EA are too
small, then the class $\mathcal{D}_\Theta$ includes other processes
which are slower than the ($\mu$,$\lambda$) EA, and the corresponding
bound $t_\Theta$ cannot be accurate. Hence, the theorem is limited by
the accuracy at which one can describe the process by some class
$\mathcal{D}_\Theta$. One could imagine a more accurate variant of the
theorem requiring more precise, and possibly harder to obtain,
information about the process, such as the variance of $D(P_t)$.

Assuming a fixed parameterisation $\Theta$, it is possible to make a
precise statement about the tightness of the upper bound $t_\Theta$.
Theorem~\ref{thm:levels-lower-bounds} stated below is an existential
statement on the form
\begin{align}
  \exists D\in\mathcal{D}_\Theta\; \expect{T_D}\geq t'_\Theta,\label{eq:lb:2}
\end{align}
where the lower bound $t'_\Theta$ is close to the upper bound
$t_\Theta$.  Hence, given the information the theorem has about the
process through the parameterisation $\Theta$, the runtime bound is
close to optimal. More information about the process would be required
to obtain a more accurate bound on the runtime.

In some concrete cases, one can prove that the level-based theorem is
close to optimal using parallel black-box complexity theory
\cite{Droste2006BlackBox,Badkobeh2014}. From Corollary~\ref{cor:GA}
with $p_\mathrm{c} = 0$, which specialises the level-based theorem to
algorithms with unary mutation operators, one can obtain the bounds
$\bigO{n\lambda+n\ln n}$ for \onemax, and
$\bigO{n\lambda\ln(\lambda)+n^2}$ for \leadingones for appropriately
parameterised EAs. These bounds are within a
$\bigO{\ln \lambda}$-factor of the lower bounds that hold for any
parallel unbiased black-box algorithm \cite{Badkobeh2014}. For
population sizes $\lambda=\bigO{n/\ln n}$ and $\lambda=\Omega(\ln n)$,
the resulting $\bigO{n^2}$ bound on \leadingones is asymptotically
tight, because it matches the lower bound that holds for all black-box
algorithms with unary unbiased variation operators
\cite{LehreWitt2012BBAlgorithmica}.

\begin{theorem}\label{thm:levels-lower-bounds}
  Given any partition of $\mathcal{X}$ into $m$
  non-empty subsets $(A_1,\ldots,A_m)$,
  for any $z_1,\ldots,z_{m-1},\delta,\gamma_0\in(0,1)$ where
  $1\geq \gamma_0(1+\delta) \geq z_j$ for all $j\in[m-1]$,
  and $\lambda\in\mathbb{N}$,
  there exists a mapping $D$ which satisfies conditions
  (G1), (G2), and (G3), of Theorem \ref{thm:general-fitness-levels},
  such that Algorithm~\ref{algo:Algorithm1} with mapping $D$ has
  expected hitting time
  \begin{align*}
    \expect{T}\geq
    \left(\frac{2}{3\delta}\sum_{j=1}^{m-2}
    \lambda\ln\left(\frac{\gamma_0\lambda}{1+2\lambda
    z_{j}+1/\delta^2}\right)\right)+\sum_{j=1}^{m-1}\frac{1}{z_j},
  \end{align*}
  where $T:=\min\{\lambda t\in\mathbb{N}\mid |P_t\cap A_m|>0\}$.
\end{theorem}
\begin{proof}[Proof of Theorem~\ref{thm:levels-lower-bounds}]
We construct an operator $D$ which leads to the claimed lower bound.
Choose any sequence of search points $(x_1,\ldots,x_m)\in A_1\times
\cdots\times A_m$, and let the initial population of Algorithm 1 be $P_0 := (x_1,
\ldots, x_1)$, i.e., $\lambda$ copies of the search point $x_1$
belonging to the first level.

For any population
$P\in\mathcal{X}^\lambda$, let the \emph{current level} be the largest $i\in[m]$ such that
$|P\cap A_{\geq i}|\geq\gamma_0\lambda$. For any population $P$ with
current level $i<m$, define the operator $D$ for all $u\in\mathcal{X}$ by
\begin{align}
  \displaystyle{\Pr_{y\sim D(P)}(y=u)}
  & :=
        \begin{cases}
          1-\max\{(1+\delta)\gamma,z_i\} & \text{if }u=x_i\\
          \max\{(1+\delta)\gamma,z_i\}  & \text{if }u=x_{i+1},\\
          0     & \text{otherwise.}
        \end{cases}\label{eq:update-phase1}
\end{align}
where $\gamma := (1/\lambda)|P\cap A_{\geq i+1}|<\gamma_0$.

For all $t\in\mathbb{N}$, define
\begin{align*}
  T_j & := \min\{t \mid |P_t\cap A_{\geq j}|> 0 \}, \text{ for all }
        j\in[m],\text{ and}\\
  S_j & := T_{j+1}-T_j \text{ for all }j\in[m-1].
\end{align*}
Then we have
$\sum_{j=1}^{m-1} S_j=T_m-T_1=T$ because $T_1=0$. The random variable $S_j$, for $j\in[m-1]$, describes the number of
generations from the time the process has discovered the search point
$x_{j}$ until it has discovered the search point $x_{j+1}$, and we
call this \emph{phase} $j$.  We divide each phase $j$ into two
sub-phases. Let $S_j^1$ be the number of generations where
\begin{align*}
  1\leq  |P_t\cap A_{\geq j}|<\gamma_0\lambda,
\end{align*}
and call this the first sub-phase,
and let $S_j^2$ be the number of generations where
\begin{align*}
 \gamma_0\lambda \leq |P_t\cap A_{\geq j}|
  \text{ and } 0 = |P_t\cap A_{\geq j+1}|,
\end{align*}
and call this the second sub-phase. The duration of the $j$-th phase
is the sum $S_j=S_j^1+S_j^2$. 
Remark that $S_1^1 = 0$
due to the choice of the initial population $P_0$.

Note also that by the definition of operator $D$, as long as the process is
in sub-phase~1 of phase $j$, the probability of generating the search
point $x_{j+1}$ is 0. Furthermore, the process never returns to
sub-phase~1 once the process has entered sub-phase~2. To estimate the
duration of sub-phase~1, we consider the stochastic process
$(X_t)_{t\in\mathbb{N}}$ where
$X_t := |P_{T_j+t}\cap A_{\geq j}|$,
and a corresponding filtration
$\left(\filt{t}\right)_{t\in\mathbb{N}}$ where
$\filt{t}:=\sigma\left(P_1,\ldots,P_{T_j+t}\right)$.

During sub-phase 1 of phase $j>1$, it holds that $X_{t+1}\sim\bin(\lambda, p_{t+1})$,
where
$p_{t+1} =\max\left\{(1+\delta)\frac{X_t}{\lambda},z_{j-1}\right\}$.

To lower bound the expected duration of sub-phase 1, we
apply drift analysis (Lemma \ref{lemma:pol-drift-lower}) with respect
to the process $(Z_t)_{t\in\mathbb{N}}$ defined by
$  Z_t  := \ln(\lambda/R_t)$
where
$R_t:=\max\{X_t,y_j\}$
and
$
  y_j:=\max\{\lambda z_{j-1},1/\delta^2\}>1.$
Note that since $z_j<\gamma_0$ by assumption, and
$1/\delta^2<\gamma_0\lambda$ by condition
(G3), it holds that $y_j<\gamma_0\lambda$.
It is therefore clear that sub-phase 1 is only complete if
\begin{align*}
  Z_t & \leq \ln\left(\frac{\lambda}{\gamma_0\lambda}\right) = -\ln(\gamma_0)=:a.
\end{align*}

By Jensen's inequality, the drift of this process can be bounded by
\begin{align*}
  \expect{Z_t-Z_{t+1}\mid \filt{t}}
  & = \expect{\ln\left(\frac{R_{t+1}}{R_t}\right)\mid \filt{t}}\\
  & \leq \ln\left(\frac{\expect{R_{t+1}\mid \filt{t}}}{R_t}\right)\\
  \intertext{and by Lemma \ref{lemma:trunc-bin-bound}}
  & \leq \ln\left(\frac{\max(\lambda p_{t+1}, y_j)+(\frac{1}{2})\sqrt{\lambda p_{t+1}}}{R_t}\right).
\end{align*}
When $\lambda p_{t+1}\geq y_j$,
we use that $R_t=\max\{X_t,y_j\}\geq \lambda p_{t+1}/(1+\delta)$ because
$p_{t+1}=\max\{X_t(1+\delta)/\lambda,z_{j-1}\}$, so
\begin{align*}
  \expect{Z_t-Z_{t+1}\mid \filt{t}}
  & \leq \ln\left(\frac{\lambda p_{t+1}+(1/2)\sqrt{\lambda p_{t+1}}}{\lambda p_{t+1}/(1+\delta)}\right)\\
  & =  \ln\left((1+\delta)\left(1+\frac{1}{2\sqrt{\lambda
    p_{t+1}}}\right)\right).
\end{align*}
Since $\lambda p_{t+1}\geq y_j>1$, thus $\sqrt{1/(\lambda p_{t+1})} \leq \sqrt{1/y_j}
    \leq \delta$  and
\begin{align*}
\expect{Z_t-Z_{t+1}\mid \filt{t}}
  & \leq  \ln(1+\delta)+\ln(1+\delta/2)
    < (3/2)\delta.
\end{align*}
Otherwise, when $y_j>\lambda p_{t+1}$, we use that $R_t\geq
y_j$ and get
\begin{align*}
  \expect{Z_t-Z_{t+1}\mid \filt{t}}
  & \leq \ln\left(\frac{y_j+(1/2)\sqrt{\lambda p_{t+1}}}{y_j}\right)\\
  & < \ln\left(1+\frac{\sqrt{y_j}}{2y_j}\right)
    = \ln\left(1+\frac{1}{2\sqrt{y_j}}\right)\\
  & \leq \ln\left(1+\frac{\delta}{2}\right) < \delta/2.
\end{align*}
Hence, condition 1 in Lemma~\ref{lemma:pol-drift-lower} can be satisfied
with the parameter $\varepsilon:=(3/2)\delta$. We therefore get the bound
\begin{align*}
  \expect{S^1_j\mid \filt{0}}\geq
  & \frac{Z_{0} - a}{\varepsilon}
   = \left(\frac{2}{3\delta}\right)\ln\left(\frac{\gamma_0\lambda}{\max\{X_{0},\lambda z_{j-1},\frac{1}{\delta^2}\}}\right).
\end{align*}
By the definition of the process, for $1<j\leq m,$ we have $X_{0}\sim (Y\mid Y\geq 1)$
where $Y\sim \bin(\lambda,z_{j-1}),$ i.e., $X_{0}$ is binomially
distributed random variable conditional on having value at least 1.
and by the tower property of expectation,
\begin{align*}
  \expect{S^1_j}
  & = \expect{\expect{S^1_j\mid \filt{0}}}\\
  & \geq
    \expect{\left(\frac{2}{3\delta}\right)\ln\left(\frac{\gamma_0\lambda}{\max\{X_{0},\lambda z_{j-1},1/\delta^2\}}\right)}\\
  & >
    \expect{\left(\frac{2}{3\delta}\right)\ln\left(\frac{\gamma_0\lambda}{X_{0}+\lambda z_{j-1}+1/\delta^2}\right)}\\
  \intertext{and since the function $f(x)=\ln(1/x)$ is convex,
  Jensen's inequality
  and
  Lemma~\ref{lemma:expect-zero-trunc-binomial} give
}
  & >
    \left(\frac{2}{3\delta}\right)\ln\left(\frac{\gamma_0\lambda}{1+2\lambda z_{j-1}+1/\delta^2}\right).
\end{align*}

During sub-phase 2, it holds that
\begin{align*}
  \displaystyle{\Pr_{y\sim D(P_t)}(y=x_j)} = 1-z_{j},\text{ and }
  \displaystyle{\Pr_{y\sim D(P_t)}(y=x_{j+1})}  = z_{j}.
\end{align*}
In each generation of sub-phase 2, the phase ends with probability
$q_j:=1-(1-z_{j})^\lambda<\lambda z_{j}$, i.e., the probability that at least
one individual is produced in $A_{\geq j+1}$. The duration of
sub-phase 2 is therefore geometrically distributed with parameter $q_j$ and
has expectation $\expect{S_j^2}=1/q_j\geq 1/(\lambda z_{j})$.

Hence, we get
\begin{align*}
  \expect{T}
  & = \sum_{j=1}^{m-1} \expect{S^1_j} + \expect{S^2_j}\\
  & \geq \left(\frac{2}{3\delta}\sum_{j=1}^{m-2}\ln\left(\frac{\gamma_0\lambda}{1+2\lambda z_{j}+1/\delta^2}\right)\right)
    + \sum_{j=1}^{m-1}\frac{1}{\lambda z_j}.\qedhere
\end{align*}
\end{proof}

\section{Conclusion}\label{sec:concl}

Time-complexity analysis of evolutionary algorithms (EAs) has advanced
significantly over the last decade, starting from simplified settings
such as variants of the (1+1) EA without a real population, crossover
or other higher-arity operators.  It has been unclear to what extent
the time-complexity profiles of these simple EAs considered by
theoreticians deviate from those of the more sophisticated,
population-based EAs often preferred by practitioners. New techniques
tailored to time-complexity of population-based algorithms are
required.

This paper introduces a new technique that easily yields upper bounds
on the expected runtime of complex, non-elitist search processes. The
technique is first illustrated on Genetic Algorithms. We have
shown that GAs optimise standard benchmark functions, as well as
combinatorial optimisation problems, efficiently. As long as the
population size is not overly large, the population does not incur an
asymptotic slowdown on these functions compared to standard EAs that
do not use populations. Thus, speedups can be achieved by
parallellising fitness evaluations. Furthermore, consequent work
indicate that non-elitist, population-based EAs have an advantage on
more complex problems, including those with noisy
\cite{DangLehre2015Noise}, dynamic \cite{DangJansenLehre2015Dynamic},
and peaked \cite{DangLehre2016SelfAdaptationArxiv} fitness landscapes.

As a side-effect of the analysis, the conditions of level-based
theorem yield settings for algorithmic parameters, such as population
size, mutation and crossover rates, selection pressure etc., that are
sufficient to guarantee a given time-complexity bound. This opens up
the possibility of theory-led design of EAs with guaranteed runtime,
where the algorithm is designed to satisfy the conditions of the 
level-based theorem \cite{CorusLehre2015MIC}.

Further demonstrating the generality of the theorem, we also provide
time-complexity results for the UMDA algorithm, an Estimation of
Distribution Algorithm, for which there are few theoretical
results. Finally, we show via lower bounds on the runtime of a
concrete process, that given the information the theorem requires
about the process, the upper bounds are close to tight, i.e.,
little improvement of the time bound in the theorem is possible.

\section*{Acknowledgment}
The research was supported by the European Union Seventh Framework
Programme (FP7/2007-2013) under grant agreement no 618091 (SAGE)
and Russian Foundation for Basic Research grants~15-01-00785 and
16-01-00740. Early ideas were discussed at Dagstuhl Seminars 13271
and 15211 ``Theory of Evolutionary Algorithms''.

\appendices

\section{}
The following results are known in the literature.

\begin{lemma}[Lemma~33 in \cite{bib:dl16}]\label{lemma:ln-func-bound}
For all $x\geq0$, $x \geq \ln(1+x) \geq x(1-x/2)$.
\end{lemma}

\begin{lemma}[Lemma~31 in \cite{bib:dl16}]\label{lemma:exp-ineq}
  For $n \in \mathbb{N}$ and $x\geq 0$, we have $1 - (1 - x)^n \geq 1 - e^{-xn} \geq \frac{xn}{1 + xn}$.
\end{lemma}

\begin{lemma}[Lemma 3 in \cite{DangLehre2016SelfAdaptationArxiv}]\label{lem:prob-no-flip}
  For any $\varepsilon\in(0,1)$ and $\chi>0$, if $n\geq
  (\chi+\varepsilon)\frac{\chi}{\varepsilon}$ then
    $
   (1-\varepsilon)e^{-\chi} \leq \left(1-\frac{\chi}{n}\right)^n \leq e^{-\chi}.
  $
  \end{lemma}

\begin{lemma}[Corollary~3 in \cite{DangLehre2015}]\label{lem:feige-ineq}
Let $Y_1,\dots,Y_n$ be $n$ independent random variables with
support in $[0,1]$ and finite expectations and $Y: =
\sum_{i=1}^{n} Y_i$. It holds for every $\delta>0$ that
\begin{align*}
  \prob{Y > \expect{Y} - \delta} \geq \min\left\{\frac{1}{13},\frac{\delta}{1+\delta}\right\}.
\end{align*}
\end{lemma}

\section{}

The following lemmas are part of the proof of Theorem~\ref{thm:general-fitness-levels}.

\begin{lemma}\label{lemma:level-functions}
  The functions $g_1$ and $g_2$ defined below are {\fprop}s for any $c>0$, $\kappa \in (0,1)$, $x \in [\lambda]$, $y \in [m]$
  and $\gamma_j,q_j\in(0,1]$ for each $j\in[m-1]$.
  \begin{align*}
    g_1(x,y)
    &:= \ln\left(\frac{1+c\lambda}{1+c\max\{x,\gamma_y\lambda\}}\right) +
        \sum_{i=y+1}^{m-1}\ln\left(\frac{1+c\lambda}{1+c\gamma_i\lambda}\right) \\
    g_2(x,y)
    &:= \begin{cases}
    \displaystyle \frac{(1 - \kappa)^x}{q_{y}} + \sum^{m-1}_{i=y+1} \frac{1}{q_i} & \text{ if } y \in [m-1],\\
    0 & \text{ if } y = m
    \end{cases}
  \end{align*}
  and $g_1(x,j):=g_2(x,j):=0$ for $j=m$.
\end{lemma}
\begin{proof}
  Both $g_1$ and $g_2$ are non-increasing functions in $x$ and $y$, hence
  properties 1 and 2 of Definition \ref{def:property} are satisfied.
  Property~3 is satisfied because for all $y \in [m-1]$
  \begin{align*}
    g_1(\lambda,y)
    & = \sum_{i=y+1}^{m-1}\ln\left(\frac{1+c\lambda}{1+c\gamma_i\lambda}\right)\\
    & = \ln\left(\frac{1+c\lambda}{1+c\gamma_{y+1}\lambda}\right) +
        \sum_{i=y+2}^{m-1}\ln\left(\frac{1+c\lambda}{1+c\gamma_i\lambda}\right)\\
    & = g_1(0,y+1)
  \end{align*}
  and
  \begin{align*}
    g_2(\lambda,y)
    & =    \frac{(1-\kappa)^\lambda}{q_{y}} + \sum^{m-1}_{i=y+1} \frac{1}{q_i}
      >    \sum^{m-1}_{i=y+1} \frac{1}{q_i} \\
    & =    \frac{(1-\kappa)^0}{q_{y+1}}+\sum^{m-1}_{i=y+2} \frac{1}{q_i}
      =  g_2(0,y+1).\qedhere
  \end{align*}
\end{proof}

Drift analysis \cite{bib:Hajek1982,bib:h01} is an important tool in runtime analysis
of randomised search heuristics. Here we introduce a variant of the additive
drift theorem \cite{bib:h01} with its proof.

In the following,
``$(X_{t+1} - X_{t} + \varepsilon) \;; t<T_{a}$'' is the short-hand
notation for ``$(X_{t+1} - X_{t} + \varepsilon) \cdot \indf{t<T_{a}}$''
(see page 49 in \cite{bib:Kallenberg2002}). Whenever
we write an equality or inequality involving conditional expectation w.r.t. a
$\sigma$\nobreakdash-algebra, (e.g. $\expect{X\mid \mathscr{F}}\leq
Y$), we have the ``almost surely'' meaning in mind.

\begin{lemma}[Additive drift theorem]\label{lemma:pol-drift}
Let $(Z_t)_{t\in\mathbb{N}}$ be a discrete-time stochastic process in
$[0,\infty)$ adapted to any filtration $(\filt{t})_{t\in\mathbb{N}}$.
For any $a\geq 0$, define $T_a := \min\{t\in\mathbb{N} \mid Z_t \leq a\}$.
If for some $\varepsilon>0$
\begin{enumerate}
  \item $\expect{Z_{t+1} - Z_{t} + \varepsilon \;; t<T_{a} \mid
      \filt{t}} \leq 0$ for all $t\in\mathbb{N}$,
  \item $\expect{Z_t}<\infty$ for all $t\in\mathbb{N}$, and
  \item $\expect{T_a } < \infty$,
\end{enumerate}
then $\expect{T_a \mid \filt{0}} \leq Z_0 / \varepsilon$.
\end{lemma}
\begin{proof}
Define the stopped process $S_t := Z_{t\wedge T_a} + \varepsilon
(t\wedge T_a)$ where $t\wedge T_a
:= \min(t, T_a)$. By the definition of this process, it holds for
all $t\in\mathbb{N}$ almost surely that
\begin{align}
  |S_t| \leq Z_t + \varepsilon t\label{drifteq:3},
\end{align}
and, hence by condition 2 and 3, that for all $t\in\mathbb{N}$,
\begin{align}
  \expect{|S_t|} \leq \expect{Z_t + \varepsilon T_a} < \infty.\label{drifteq:1}
\end{align}
Also, by the definition of the process, for all $t\in\mathbb{N}$ it
holds in the case $t\geq T_a$ that,
\begin{align*}
  \expect{S_{t+1}\;; t\geq T_a \mid\filt{t}}=\expect{S_{t}\;; t\geq T_a \mid\filt{t}}.
\end{align*}
Furthermore, for all $t\in\mathbb{N}$, it holds in the case $t< T_a$,
\begin{align*}
  & \expect{S_{t+1}\;; t< T_a\mid\filt{t}} \\
    &\quad =     \expect{(Z_{t+1}-Z_t+\varepsilon)+Z_t+\varepsilon t\;; t< T_a\mid\filt{t}} \\
    &\quad \leq  \expect{Z_t+\varepsilon t\;; t< T_a\mid\filt{t}}         =    \expect{S_t\;; t< T_a\mid\filt{t}}
\end{align*}
where the inequality is due to condition 1. Combining both cases, we have for all $t\in\mathbb{N}$,
\begin{align}
  \expect{S_{t+1} \mid\filt{t}}
  & \leq   \expect{S_{t} \mid\filt{t}} = S_t.\label{drifteq:2}
\end{align}
By (\ref{drifteq:1}) and (\ref{drifteq:2}), $S_t$ is a super-martingale, implying that for all $t\in\mathbb{N}$,
\begin{align}
 \expect{S_t\mid\filt{0}} \leq \expect{S_0 \mid \filt{0}} = Z_0.\label{drifteq:4}
\end{align}
By (\ref{drifteq:3}) and (\ref{drifteq:1}), the dominated convergence
theorem (see e.g. \cite{bib:Kallenberg2002}) applies, and we get by (\ref{drifteq:4})
\begin{align*}
  Z_0   \geq \lim_{t\rightarrow\infty}\expect{S_t\mid\filt{0}}
        & = \expect{ \lim_{t\rightarrow\infty} S_t\mid\filt{0}}\\
        & = \expect{ Z_{T_a} +\varepsilon T_a\mid\filt{0} }.
\end{align*}
By noting that $Z_{T_a}\geq 0$, the proof is now complete.
\end{proof}

\begin{lemma}[Improved version of Lemma 5 in \cite{bib:dl16}]\label{appendix:lemma:ln-expect-bound}
  If $X\sim\bin(\lambda,p)$ with $p\geq (i/\lambda)(1+\delta)$ and $i \geq 1$
  for some $\delta \in (0,1]$, then
  \begin{align*}
    \expect{\ln\left(\frac{1+\delta X/2}{1+ \delta i/2}\right)} \geq \frac{\delta^{2}}{7}.
  \end{align*}
\end{lemma}

This improvement is due to the following generalisation of the lower bound in
Lemma~\ref{lemma:ln-func-bound}.

\begin{lemma}\label{lem:bound-ln-func}
For any $z > 0$, and all $x\geq 0$ we have that
\begin{gather*}
  \ln(1 + x) \geq x(b(z) + a(z)x) \\
  \text{ where } a(z) := \frac{1}{z(z+1)} - \frac{\ln(1+z)}{z^2}, \\
  \text{ and } b(z) := \frac{2\ln(1+z)}{z} - \frac{1}{1+z}. \nonumber
\end{gather*}
\end{lemma}
\begin{proof}
For $x=0$, the result trivially holds. It then suffices to show that for all $x
\in (0,\infty)$
\begin{align*}
  h(x) := \frac{\ln(1+x)}{x} - b(z) - a(z)x \geq 0.
\end{align*}

Note that $h(z) = 0$
and $h'(x)=a(x) - a(z)$.
It follows from $\ln(1+x) > 2x/(x+2)$ for $x>0$ (see (3) in \cite{bib:Topsoe2007}) that
\begin{align*}
  a'(x) &= \frac{2\ln(1+x)}{x^3} - \frac{2}{x^2(1+x)} - \frac{1}{x(1+x)^2} \\
        &\geq \frac{4}{x^2(2+x)} - \frac{2}{x^2(1+x)} - \frac{1}{x(1+x)^2} \\
        &= \frac{1}{(2+x) (1+x)^2} > 0,
\end{align*}
thus $a(x)$ is an increasing function.

We separate two cases: for $x \in (0,z]$, we have $a(x) \leq a(z)$ and $h'(x)
\leq 0$, thus $h(x)$ is decreasing on $(0,z]$ and $h(x) \geq h(z) = 0$; for $x
\in [z,\infty)$ we have $h'(x) = a(x) - a(z) \geq 0$, $h(x)$ is increasing on
$[z,\infty)$ and $h(x) \geq h(z) = 0$. We have shown that $h(x) \geq 0$ for
$x>0$.
\end{proof}

Note that the bound is tight at both $x=0$ and $x=z$. The lemma does not cover
the case $z=0$, however at the limit, we get $\lim_{z \rightarrow 0^+} b(z) = 1$
and $\lim_{z \rightarrow 0^+} a(z) = -1/2$, and that corresponds to the bound
given by Lemma~\ref{lemma:ln-func-bound}.
\begin{corollary}\label{cor:bound-expect-ln-scale}
Let $X \sim \bin(n,p)$ and $\mu := \expect{X}$, then it holds that for all $c>0$
\begin{align*}
  \expect{\ln(1+cX)}
    \geq \ln(1 + c\mu) - \frac{c}{2} \cdot \frac{c\mu}{1+c\mu}.
\end{align*}
\end{corollary}
\begin{proof}
For $p=0$ (or $\mu = 0$), the bound is trivial. Otherwise, for $p>0$, applying
Lemma \ref{lem:bound-ln-func} with $z=c\mu$ gives $\ln(1+cX) \geq b(c\mu)cX +
a(c\mu)(cX)^2$, hence
\begin{align*}
  &\expect{\ln(1+cX)} \\
    &\geq           b(c\mu)c\mu + a(c\mu)c^2 \mu(1 - p + \mu) \\
                    &= \ln(1 + c\mu) - c(1-p)\left(\frac{\ln(1+c\mu)}{c\mu} - \frac{1}{1+c\mu}\right) \\
    &>    \ln(1 + c\mu) - c\left(\frac{1}{2}\cdot\frac{2+c\mu}{1+c\mu} - \frac{1}{1+c\mu}\right) \\
    &=    \ln(1 + c\mu) - \frac{c}{2} \cdot \frac{c\mu}{1+c\mu}.
\end{align*}
The last inequality is due to $1-p<1$ and $\ln(1+x)/x < (1/2)(x+2)/(x+1)$ for $x > 0$
(see (3) in \cite{bib:Topsoe2007}).
\end{proof}

We now give the formal proof of Lemma~\ref{appendix:lemma:ln-expect-bound}.
\begin{proof}[Proof of Lemma~\ref{appendix:lemma:ln-expect-bound}]
Let $Y\sim\bin(\lambda, (1+\delta)i/\lambda)$, then $Y \preceq X$. Therefore,
\[
  \expect{\ln\left(\frac{1+\delta X/2}{1+\delta i/2}\right)}
    \geq \expect{\ln\left(\frac{1+\delta Y/2}{1+\delta i/2}\right)}
\] and it is sufficient to show that $\expect{\ln\left(\displaystyle \frac{1+\delta Y/2}{1+\delta i/2}\right)} > \delta^2/7$ to complete the proof.

It follows from Corollary \ref{cor:bound-expect-ln-scale} (choosing $c=\delta/2$) that
\begin{align*}
  &\expect{\ln\left(\frac{1 + \delta Y/2}{1 + \delta i/2}\right)} \\
    &\geq \ln\left( \frac{1 + (1+\delta) \delta i/2}{1 + \delta i/2} \right)
          - \frac{\delta}{4} \cdot \frac{(1+\delta) \delta i/2}{1 + (1+\delta) \delta i/2} \\
    &=    \ln\left( 1 + \frac{i \delta^2}{2 + i\delta} \right)
          - \frac{\delta}{4} \cdot \frac{(1+\delta) \delta i}{2 + (1+\delta) \delta i} =: h(i).
\end{align*}

For all $\delta>0$ and $i \geq 1$, it holds that
\begin{align*}
  h'(i)
    = \frac{1}{2} \cdot \frac{(6 + \delta(3 i - 2) + 3 i \delta^2)\delta^2}
                             {(2+i\delta + i \delta^2)^2 (2 + i \delta)} > 0,
\end{align*}
or $h(i)$ monotonically increases in $i$.

Define $r(\delta):=12 + 8\delta + 3\delta^2 + \delta^3 - 2\delta^4$ and
$s(\delta):=(2+\delta)^2 (2 + \delta + \delta^2)>0$,
we get
\begin{align*}
  h(i)
    &\geq h(1)
     =     \ln\left( 1 + \frac{\delta^2}{2 + \delta} \right)
          - \frac{\delta}{4} \cdot \frac{(1+\delta) \delta }{2 + (1+\delta) \delta } \\
    &\geq \frac{\delta^2}{2 + \delta} \left(1 - \frac{\delta^2}{2(2 + \delta)}\right)
          - \frac{\delta}{4} \cdot \frac{(1+\delta) \delta }{2 + (1+\delta) \delta }              = \frac{\delta^2 r(\delta)}{4 s(\delta)}.
\end{align*}
The last inequality is due to Lemma~\ref{lemma:ln-func-bound}. We notice that
$18 r(\delta) - 11 s(\delta) = (1-\delta)(128 + 140\delta + 84 \delta^2 +
47\delta^3) \geq 0$ for all $\delta \in (0,1]$, thus $r(\delta)/s(\delta) \geq
11/18$ and $h(i) \geq (\delta^2/4) (11/18) > \delta^2/7$.
\end{proof}

\begin{lemma}[Lemma 6 in \cite{bib:dl16}]\label{appendix:lemma:neg-moment-binomial}
  If $X\sim\bin(\lambda,p)$ with $p\geq (i/\lambda)(1+\delta)$, then $\expect{e^{-\kappa X}} \leq e^{-\kappa i}$ for any $\kappa\in(0,\delta)$.
\end{lemma}
\begin{proof}
  The value of the moment generating function $M_X(t)$ of the binomially distributed variable $X$ at $t=-\kappa$ is
  \begin{align*}
    \expect{e^{-\kappa X}} = M_X(-\kappa) = (1-p(1-e^{-\kappa}))^\lambda
  \end{align*}
  It follows from By Lemma \ref{lemma:exp-ineq} and from $1+\kappa<1+\delta$ that
  \begin{align*}
    p(1-e^{-\kappa})
    \geq \frac{i(1+\delta)}{\lambda}\left(\frac{\kappa}{1+\kappa}\right)
    \geq \frac{\kappa i}{\lambda}.
  \end{align*}
  Altogether, we get     $\expect{e^{-\kappa X}}
   \leq (1-\kappa i/\lambda)^\lambda
    \leq e^{-\kappa i}$.   \end{proof}

\begin{lemma}\label{appendix:lemma:skip-condition-in-prob}
Let $\{X_i\}_{i\in[\lambda]}$ be i.i.d. random variables, define $Y(j):=
\sum_{i=1}^{\lambda} \indf{X_i \geq j}$ for any $j \in \mathbb{R}$. It holds
for any $a,b,c,j \in \mathbb{R}$ with $c\geq 0$ and $b \leq \lambda$ that
\begin{align*}
(i)  &\  \prob{Y(j+c) \geq a \mid Y(j) \geq b} \geq
\prob{Y(j+c)\geq a} \intertext{and for any non-decreasing
function~$f$} (ii) &\  \expect{f(Y(j+c)) \mid Y(j) \geq b} \geq
\expect{f(Y(j+c))}
\end{align*}
provided that both expectations are well-defined.
\end{lemma}
\begin{proof}
Define $p := \prob{X_i \geq j}$ and $q := \prob{X_i \geq j+c}$.
For $b \leq 0$ or $p = 0$, the result trivially holds. For $b \in
(0,\lambda]$ and $p \in (0,1]$, we have that $q' := \prob{ X_i
\geq j+c \mid X_i \geq j} = q/p \geq q$. Event $Y(j)\geq b$
implies the existence of a set $A \subseteq [\lambda]$ such that
$|A| \geq \lceil b\rceil$ and $X_i \geq j$ for all $i \in A$.
Define $Y_1 := \sum_{i\in A} \indf{X_i \geq j + c}$ and $Y_2 :=
\sum_{i\in [\lambda] \setminus A} \indf{X_i \geq j + c}$, so
$Y(j+c) = Y_1 + Y_2$. Clearly, conditioned on $Y(j)\geq b$, $Y_1
\sim \bin(|A|,q') \succeq \bin(|A|,q)$ and $Y_2 \sim \bin(\lambda
- |A|,q)$. Therefore, the distribution of $Y(j+c)$ conditioned on
$Y(j)\geq b$ stochastically dominates $\bin(|A|,q) + \bin(\lambda
- |A|,q) = \bin(\lambda,q)$, which is the (unconditional or
original) distribution of $Y(j+c)$, and part~(i) follows.

For part~(ii), let $F_1(x):=\prob{f(Y(j+c)) < x \mid Y(j) \geq b}$
and $F_2(x):=\prob{f(Y(j+c)) < x},$ i.e. $F_1$ and $F_2$ are the
conditional and the unconditional distribution functions
of~$f(Y(j+c))$ respectively. Then from part~(i) we conclude that
$F_{1}(x) \le F_{2}(x)$ for any $x \in \mathbb{R}$, and by the
properties of expectation,
\begin{align*}
  & \expect{f(Y(j+c)) \mid Y(j) \geq b}\\
   &   = -\int_{-\infty}^0 F_1(x) dx + \int_0^{\infty}
(1-F_1(x)) dx \\
   & \ge -\int_{-\infty}^0 F_2(x) dx + \int_0^{\infty}
(1-F_2(x)) dx\\
   & = \expect{f(Y(j+c)) \mid Y(j) \geq b}.\qedhere
\end{align*}
\end{proof}

\section{}
The following results are used to analyse the tightness of
the level-based theorem, \ie Theorem~\ref{thm:levels-lower-bounds}.

\begin{lemma}[Additive drift theorem (lower bound)]\label{lemma:pol-drift-lower}
Let $(Z_t)_{t\in\mathbb{N}}$ be a discrete-time stochastic process in
$[0,\infty)$ adapted to any filtration $(\filt{t})_{t\in\mathbb{N}}$.
For any $a\geq 0$, define $T_a := \min\{t\in\mathbb{N} \mid Z_{t} \leq a\}$.
If for some $\varepsilon>0$
\begin{enumerate}
  \item $\expect{Z_{t+1} - Z_{t} + \varepsilon \;; t<T_{a} \mid \filt{t}} \geq 0$ for all $t\in\mathbb{N}$, and
  \item $\expect{Z_t}<\infty$ for all $t\in\mathbb{N}$.
  \item $\expect{T_a} < \infty$,
\end{enumerate}
then $\expect{T_a \mid \filt{0}} \geq (Z_0-a) / \varepsilon$.
\end{lemma}
\begin{proof}
The proof is similar to that of Lemma~\ref{lemma:pol-drift}, \ie starting
by defining the same stopped process $S_t$. However, because the directions
of the inequalities are inverted so $S_t$ is a sub-martingale, and in the end
we overestimate $X_{T_a}$ by $a$.
\end{proof}

\begin{lemma}\label{lemma:trunc-bin-bound}
If $X\sim\bin(n,p)$ where $p>0$, then for all $y\in\mathbb{R}$
\begin{align*}
  \expect{\max(X,y)} < \max(np,y)+(1/2)\sqrt{np}.
\end{align*}
\end{lemma}
\begin{proof}
By Jensen's inequality \wrt the square root, we have \begin{align*}
  \expect{|X-y|}
   = \expect{\sqrt{(X-y)^2}}
  & \leq \sqrt{\expect{(X-y)^2}}\\
  & = \sqrt{np(1-p)+(np-y)^2}\\
  & \leq \sqrt{np(1-p)}+|np-y|,
\end{align*}
where the last inequality uses $\sqrt{a+b}\leq \sqrt{a}+\sqrt{b}$ for
$a,b\geq 0$. Therefore, it holds that
\begin{align*}
  \expect{\max(X,y)}
      &=    \expect{(1/2)(X + y+ |X-y|)}\\
  &\leq    (1/2)(np+y + |np-y|+\sqrt{np(1-p)})\\
  &<    \max(np,y)+(1/2)\sqrt{np}.\qedhere
\end{align*}
\end{proof}

\begin{lemma}\label{lemma:expect-zero-trunc-binomial}
  If $X\sim\bin(n,p)$ where $p>0$ then it holds that
  $\expect{X\mid X>0}\leq np+1$.
\end{lemma}
\begin{proof}
By definition,
\begin{align*}
    \expect{X\mid X>0}
    & = \sum_{i=1}^n i\prob{X=i\mid X>0}\\
    & = \frac{1}{\prob{X>0}}\sum_{i=1}^n i\prob{X=i\cap X>0}\\
    & = \frac{1}{\prob{X>0}}\sum_{i=0}^n i\prob{X=i}\\
    & = \frac{\expect{X}}{\prob{X>0}}       =    \frac{np}{1-(1-p)^n}       \leq np+1, \end{align*}
where the last inequality follows from Lemma~\ref{lemma:exp-ineq}.
\end{proof}



\end{document}